\documentclass{article}

\usepackage[preprint, nonatbib]{neurips_2022}

\usepackage{mystyle}

\title{Interactive Recommendations for Optimal Allocations in Markets with Constraints}

\author{
  Yigit Efe Erginbas \textsuperscript{1} \thanks{equal contribution} , Soham Phade \textsuperscript{2} \samethanks , Kannan Ramchandran \textsuperscript{1} \\
  \textsuperscript{1} Department of Electrical Engineering and Computer Science\\
  University of California, Berkeley\\
   \textsuperscript{2}  Salesforce Research \\ 
  \texttt{\{erginbas,soham\_phade,kannanr\}@berkeley.edu}\\
}

\begin{document}

\maketitle


\begin{abstract}

Recommendation systems when employed in markets play a dual role: they assist users in selecting their most desired items from a large pool and they help in allocating a limited number of items to the users who desire them the most. Despite the prevalence of capacity constraints on allocations in many real-world recommendation settings, a principled way of incorporating them in the design of these systems has been lacking. Motivated by this, we propose an interactive framework where the system provider can enhance the quality of recommendations to the users by opportunistically exploring allocations that maximize user rewards and respect the capacity constraints using appropriate pricing mechanisms. We model the problem as an instance of a low-rank combinatorial multi-armed bandit problem with selection constraints on the arms. We employ an integrated approach using techniques from collaborative filtering, combinatorial bandits, and optimal resource allocation to provide an algorithm that provably achieves sub-linear regret, namely \smash{$\widetilde{\mathcal{O}} ( \sqrt{N M (N+M) RT} )$} in $T$ rounds for a problem with $N$ users, $M$ items and rank $R$ mean reward matrix. Empirical studies on synthetic and real-world data also demonstrate the effectiveness and performance of our approach.

\end{abstract}


\section{Introduction}




Online recommendation systems have become an integral part of our socioeconomic life with the rapid increases in online services that help users discover options matching their preferences. 
Despite providing efficient ways to discover information about the preferences of users, they have played a largely complementary role to searching and browsing with little consideration of the accompanying \emph{markets} within which recommended items are allocated to the users. Indeed, in many real-world scenarios, recommendations bring about the \emph{allocation} of the corresponding items in a market that has possibly intrinsic constraints. 
In particular, recommendations of candidate items that have associated notions of limited \emph{capacities} naturally give rise to a market setting where users compete for the allocation of the recommended items.

Allocation constraints are common in recommendation contexts.
A few interesting examples include: (1) Point-of-Interest (PoI) recommendation systems (e.g., restaurants, theme parks, hotels), where the PoI can only accomodate limited number of visitors, (2) book recommendation systems employed by libraries, where the books recommended to the borrowers have limited copies, (3) route recommendation systems which aim to suggest the optimal road for travelling while avoiding traffic congestion, (4) course recommendation systems for universities, where each recommended course has limited number of seats. As similar systems become more ubiquitous and impactful in the broader aspects of daily life, there is a huge application drive and potential for delivering recommendations that respect the requirements of the market. Therefore, it is crucial to consider capacity-aware recommendation systems to maximize the user experience. 

\textbf{Main Challenges: }We model the user preferences as rewards that users obtain by consuming different items, while the social welfare is the aggregate reward over the entire system comprising multiple users with heterogeneous preferences, and a provider who continually recommends items to the users and receives interactive reward feedback from them. The provider aims to maximize the social welfare while respecting the \emph{time-varying} allocation constraints: indeed we consider system \emph{dynamics} in terms of user demands and item capacities to be an important aspect of our problem. In the process of identifying the best match between users and target items, the provider encounters two challenges: The first challenge relates to the element of \emph{recommendation} as the provider needs to make recommendations without exact knowledge of the user preferences ahead of time, and hence has to continue exploring user preferences while continually making recommendations. The second challenge relates to the \emph{allocation} aspect of the problem induced by the market constraints. Note that even if matching the users with their most preferred items would result in high rewards, such an allocation may not respect the constraints of the market. For example, in a restaurant recommendation setting, if there is a hugely popular restaurant that most people love, a naive recommender would send many users to the same restaurant, causing overcrowding and considerable user dissatisfaction. 

The key to overcoming the (first) challenge of making accurate recommendations is to learn the user preferences from the reward feedback. Since the preferences of different users for different items are highly correlated, it is natural to employ collaborative filtering techniques that have been widely applied in recommender systems \cite{schafer_2007, bennett_2007, koren_2009, sarwar_01}. In order to learn the user preferences efficiently, previous works have established interactive collaborative filtering systems that query the users with well-chosen recommendations \cite{kawale_2015, zhao_2013}. Typically, these works consider a setting where a single user arrives to the system at each round and the system makes a recommendation that will match the user's preferences. However, this assumption no longer holds in applications having  an associated market structure, as recommendations made to different users in the same time period must also respect the constraints of the market.

The common strategy to tackling the (second) allocation challenge is through pricing mechanisms that ensure social optimality. Such mechanisms have been studied in economics for two-sided (supply and demand) markets and are called Walrasian auctions \cite{smith_1991}. In the networking literature, Kelly has also used similar mechanisms to do optimal bandwidth allocation over a network \cite{kelly_97}. In pricing-based mechanisms, the users choose the items based on their preferences as well as the posted prices. The provider meanwhile successively adjusts these prices in response to the user's demand for the items, so that capacity constraints are satisfied in equilibrium. The equilibrium prices ensure that the limited number of items are allocated to users that are expected to obtain the largest reward. However, these mechanisms still require the users to know and evaluate their preferences for \emph{all} possible items and respond constantly with their updated bids/demands for each item. This is definitely not a scalable solution for the large-scale system (comprising large numbers of users and items) that we target. Furthermore, this framework assumes that users already know their preferences for all items, which is clearly not true in our setting, where users report their preferences through feedback \emph{after} being targeted with their recommended items. For this reason, the provider must \emph{learn} the user preferences in its quest to perform optimal capacity-constrained allocations.

\begin{wrapfigure}{r}{0.36 \textwidth} 
\vspace{-7pt}
\includegraphics[width= 0.36 \textwidth]{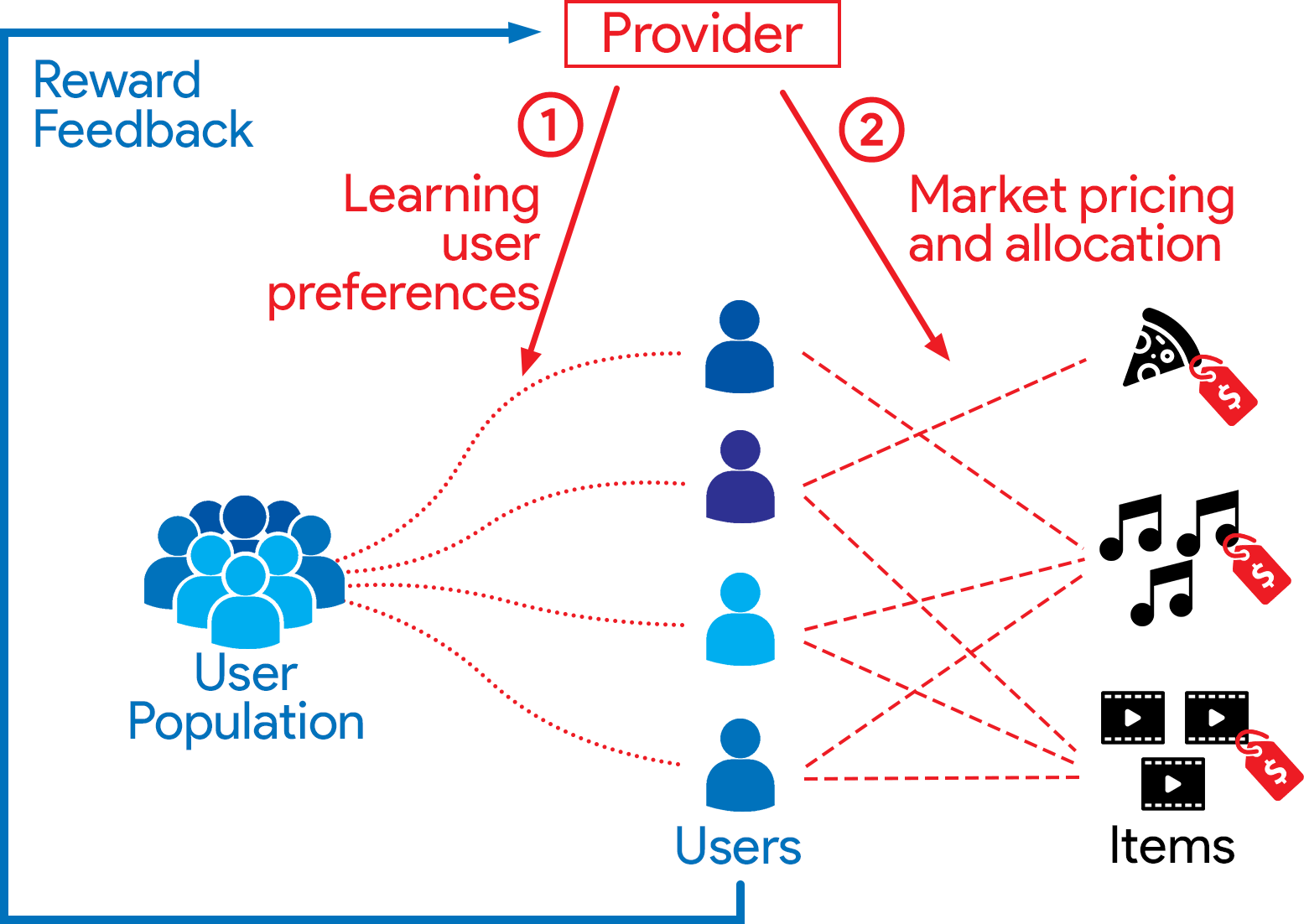}
\caption{The provider interactively learns the user preferences to achieve socially optimal capacity-constrained allocations.}
\label{system_diagram}
\vspace{-8pt}
\end{wrapfigure}
Hence, as depicted in Figure \ref{system_diagram}, the goal of the provider is twofold: (1) to learn the user preferences and make recommendations that will guide the users to choose the items that they are likely to obtain high rewards, (2) to achieve allocations that will satisfy the capacity constraints. To achieve these goals, we envision developing the following market-aware recommendation mechanism for the provider. By recommending items, the provider helps the users to narrow down their options so that users can comprehend and evaluate their preferences among a smaller number of offered items. In addition, being aware of the market structure, the provider carefully determines the item prices that play the role of an intermediary for satisfying the constraints of the market. We believe that this is an important and practically-relevant question to be resolved because it allows for the analysis of many interesting real-world interactive recommendation settings with market constraints. In its full generality, this framework requires us to model the user decisions in a way that will capture the effects of the recommendations and prices that they are presented. In order to avoid the complications introduced by this modelling challenge and to obtain a profound understanding of fundamental aspects of the problem, we begin with focusing our attention on a central question whose solution will be key to making progress towards our longer-term goal of developing a complete framework.


Specifically, we focus our study on these essential aspects of the problem: recommending and allocating the items while interactively learning the user preferences, which to the best of our knowledge has not been addressed in the literature. In essence, we analyze a special case of the mechanism introduced above, by assuming that the provider makes recommendations such that the number of presented choices matches with the number of items the user is willing to consume, so that the users obtain all of the recommended items regardless of their prices. Then, the provider's task reduces to deciding on high-reward allocations while satisfying the constraints by allocating each item to at most certain number of users.

\textbf{Structured Combinatorial Multi-Armed Bandits: } The provider seeks to choose high-reward allocations subject to the constraints, while actively learning the user preferences by making queries that will give rise to the most informative responses. Therefore, it encounters the well-known \textit{exploration-exploitation} dilemma. In essence, there exists a trade-off between two competing goals: maximizing social welfare using the historical feedback data, and gathering new information to improve the performance in the future. In the literature of interactive collaborative filtering, this dilemma is typically formulated as a multi-armed bandit problem where each arm corresponds to allocation of an item to a user \cite{zhao_2013, barraza_2017, wang_2019}. When an item is allocated to a user, a random reward is obtained and the reward information is fed back to the provider to improve its future allocation strategies. However, in contrast to prior works, our setting further requires that a collection of actions taken for different users satisfy the constraints of the market. 

We formulate our problem as a bandit problem with arms having correlated means, and call it Structured Combinatorial Bandit. Based on the standard OFU (Optimism in Face of Uncertainty) principle for linear bandits \cite{dani_2008, abbasi_2011}, we devise a procedure that learns the mean reward values opportunistically so as to solve the system problem of optimal allocation with minimum regret. The estimation method benefits from both the combinatorial nature of the feedback and the dependencies induced by the low-rank structure of collaborative filtering setting. Moreover, using matrix factorization techniques, the algorithm is efficient even at scale in settings with a large number of users and items. As is standard with OFU-based methods, our algorithm maintains a confidence set of the mean rewards for all user-item pairs. If it has less data about some user-item allocation pair, the confidence set becomes wider in the corresponding direction. Then, due to optimism, the algorithm becomes more inclined to attempt the corresponding allocation pairs to explore and collect more information. 

\textbf{Our contributions: } 
\begin{itemize}[nosep, labelindent= 0pt, align= left, labelsep=0.4em, leftmargin=*]
    \item We formulate the problem of making recommendations that will facilitate socially optimal allocation of items with constraints. Our formulation further allows for the analysis of problem settings with dynamic (i.e., time-varying) item capacities and user demands.
    \item We pose the Structured Combinatorial Bandit problem under generic structural assumptions (not only low-rank) and propose an algorithm that achieves sublinear regret bounds in terms of parameters that depend on the problem-specific structure of the arms.
    \item For the recommendation setting, we specialize our results to low-rank structures and obtain a Low-Rank Combinatorial Bandit (LR-COMB) algorithm that achieves \smash{$\widetilde{\mathcal{O}} ( \sqrt{N M (N+M) RT} )$} regret in $T$ rounds for a problem with $N$ users, $M$ items and rank $R$ mean reward matrix.
\end{itemize}
\textbf{Experiments: } We run experiments both on synthetic and real-world datasets to show the efficacy of the proposed algorithms. Results show that proposed algorithm can obtain significant improvements over naive approaches in solving the problem of recommendation and allocation with constraints.

\textbf{Related work: }

\begin{itemize}[nosep, labelindent= 0pt, align= left, labelsep=0.4em, leftmargin=*]

\item \textbf{Combinatorial Multi-Armed Bandits (CMAB) and Semi-Bandits: } The frameworks of CMAB \cite{chen_2013, kveton_2015} and semi-bandits \cite{audibert_2011} model multi-armed bandit problems where the player chooses a subset of arms in each round and observes individual outcomes of the played arms. However, they do not incorporate any structural assumptions about the rewards obtained from the arms. However, in a collaborative filtering setting like ours, the main promise is to leverage the intrinsic structure between different user-item pairs. To close this gap, we pose the problem of Structured Combinatorial Bandits and devise an algorithm that makes use of the structure of the arms as well. Additionally, CMAB framework assumes availability of an oracle that takes the means rewards for the arms and outputs the optimum subset of arms subject to the selection constraints. Due to the combinatorial nature of the problem, this oracle may not be readily available in general CMAB settings. In our case, due to the special structure of the capacity constraints, we can efficiently solve for the optimum allocations given the mean rewards of the allocation pairs.

\item \textbf{Structured Linear Bandits: } Our formulation also shows parallelism with the frameworks of structured linear bandits \cite{johnson_2016, combes_2017} and low-rank linear bandits \cite{lu_2021}. However, it is distinct from them by having the additional ability to capture the combinatorial nature of the problem. In linear bandits, the player only observes the final total reward, but no outcome of any individual arm. Our setup differs from their case because the player (provider) is able to observe individual outcomes of all played arms. Due to this richer nature of the observation model, we can achieve lower regret guarantees than what is available in the literature of structured linear bandits.

\item \textbf{Recommendation with Capacity Constraints: } There have been a few works using the notion of constrained resources to model and solve the problem of recommendation with capacity constraints \cite{christakopoulou_2017, makhijani_2019}. However, these works only consider optimizing the recommendation accuracy subject to item usage constraints without any consideration of the interactive mechanisms that discover user preferences through recommendations.

\item \textbf{Competing Bandits in Matching Markets: } One other related line of literature studies the stable matching problem in two-sided markets \cite{liu_2020}. The model assumes that entities on each side of the market has preference orderings for the other side of the market and the allocations are driven by these preference orderings rather than the prices. In contrast to our work, these mechanisms necessitate at least the entities on one side of the market know their preferences over all the entities on the other side of the market. However, in many real-world settings of optimum recommendation and allocation, like the examples given above, the explicit preferences are not known ahead of time and can only be discovered through interactions. Furthermore, the matching markets only model one-to-one matches, meaning that they do not allow for the items to be allocated for multiple users.
\end{itemize}
\vspace{-2pt}

\section{Problem setting}
\label{sec:setting}

We use bold font for vectors $\mathbf{x}$ and matrices $\mathbf{X}$, and calligraphic font $\mathcal{X}$ for sets. We denote by $[K]$ the set $\{1,2, \dots, K\}$.
For a vector $\mathbf{x}$, we denote its $i$-th entry by $x_i$ and for a matrix $\mathbf{X}$, we denote its $(i,j)$-th entry by $x_{ij}$. We denote the Frobenius inner product of two matrices by $\langle \vect{A}, \vect{B} \rangle = \tr (\vect{A}^\mathrm{T} \vect{B})$, and the Frobenious norm of a matrix $\vect{A}$ by $\|\vect{A}\|_\text{F}$.

Suppose the \emph{system} has $N$ \emph{users} and $M$ \emph{items} in record. 
The items are \emph{allocated} to the users in multiple \emph{rounds} (or \emph{periods}) denoted by $t \in \mathbb{N}$. Allocation of an item $i \in [M]$ to a user $u \in [N]$ results in a random \emph{reward} that has a distribution unknown to the system provider. The expected reward obtained from allocating item $i$ to user $u$ is denoted by $\theta^*_{ui}$ and these values are collected into the mean reward matrix $\mathbf{\Theta}^* \in \mathbb{R}^{N \times M}$. 

We assume that each item has (time-varying) capacity that corresponds to the maximum number of different users it can be allocated to. We denote the capacity of item $i \in [M]$ by $c_{t,i}$, and collect these values into vectors $\mathbf{c}_t \in \mathbb{R}^{M}$. Similarly, each user has a (time-varying) demand that corresponds to the maximum number of different items it can get allocated. We denote the demand of user $u \in [N]$ by $d_{t,u}$, and collect these values into vectors $\mathbf{d}_t \in \mathbb{R}^{N}$. Therefore, each item can only be allocated to at most $c_{t,i}$ different users, while each user can only get allocated at most $d_{t,u}$ different types of items in the period $t$. We shall call these the \emph{allocation constraints}. One can consider the special case where $d_{t,u}$ parameters only take values from $\{0, 1\}$ so that each \emph{active} user gets at most one allocation while the \emph{inactive} users do not get any allocations.

Let $\mathbf{X}_{t}$ denote the \emph{allocation matrix} for round $t$ where the $(u, i)$-th entry is one if user $u$ is allocated item $i$ at round $t$, and zero otherwise. Due to the allocation constraints, any valid $\mathbf{X}_{t}$ must belong to the set of valid allocation matrices $\mathcal{X}_t \subseteq \{0, 1\}^{N \times M}$ defined as:
\begin{equation*}
    \mathcal{X}_t = \{ \mathbf{X} \in \{0, 1\}^{N \times M} : \mathbf{X} \mathds{1}_M \leq \mathbf{d}_t \text{ and } \mathbf{X}^\textrm{T} \mathds{1}_N \leq \mathbf{c}_t\}
\end{equation*}
\vspace{-1pt}
where the inequalities are entry-wise and $\mathds{1}_p$ denotes the all-ones vector of size $p$.

\subsection{Optimal allocations}
\label{sect_opt_allocations}

Given the knowledge about the mean reward matrix $\vect{\Theta}^*$, the optimal allocation $\vect{X}^*_t$ at time $t$ can be obtained by solving the integer program:
\begin{equation}
    \vect{X}^*_t \in \argmax_{\vect{X} \in \mathcal{X}_t} \; \langle \vect{X}, \vect{\Theta}^* \rangle
    \label{integer_num}
\end{equation}
This integer program can be relaxed to a linear program by dropping the integral constraints (setting $0 \leq x_{ui} \leq 1$). In Appendix \ref{appendix_num}, we show that the integrality gap of this problem is zero. \footnote{The integrality gap is the difference between optimal values of the integer program and its linear relaxation.} Hence, any integer solution found for the relaxed problem is also a solution for the allocation problem. 

When the provider does not have direct knowledge of the mean rewards associated with user-item allocation pairs, one standard approach is to employ pricing mechanisms \cite{smith_1991, kelly_97}. 
The idea is to apply dual decomposition on the (partial) Lagrangian function $ L( \vect{X},  \vect{\lambda}) = \langle \vect{X}, \vect{\Theta}^* \rangle + \vect{\lambda}^\textrm{T} ( \vect{c}_t - \vect{X}^\textrm{T} \mathds{1}_N)$ where $\vect{\lambda} \geq 0$ are the Lagrange multipliers (item prices) associated with the capacity constraints. Then, the allocation problem is decomposed into one problem for each user and one problem for the provider where the item prices mediates between the subsidiary problems. Each user calculates its demand by maximizing the corresponding component of the Lagrangian for a given set of prices. On the other side, the provider iteratively updates the prices based on users demands to achieve the optimal pricing. At the end of many consecutive updates from users and the provider, the equilibrium ensures that the limited items are allocated to users that are expected to obtain the largest reward. (See \cite{palomar_2006} for further details.)


However, as discussed in the introduction, this pricing mechanism has limitations in many real-world applications. Most importantly, it requires the user to solve a problem that involves the valuations even for the items that the user has no prior experience with. However, in many real-world scenarios, it is infeasible to request the users to choose among all the items in the system. Secondly, in the process of price discovery, the mechanism asks the users to repetitively respond to the prices by recomputing their demands. However, since it might take many iterations until convergence to the optimal pricing, asking the users to respond many times would be a burden for them. Furthermore, the final prices found by this iterative mechanism are only guaranteed to be optimal for the problem defined by the capacity $\vect{c}_t$ and demand $\vect{d}_t$ parameters at round $t$. If the capacities and demands vary with time, the optimal pricing and allocation for the next allocation round $t+1$ will be different and will be needed to be rediscovered.  

\subsection{Learning the optimal allocations}

\label{learning_opt_alloc}

To address the issues discussed in the previous section, we need mechanisms that can find the optimum allocations using fewer and simpler interactions. One resolution is to recommend a subset of items along with prices intelligently chosen by the provider. This way, the users will be able to easily evaluate their preference on the small number of recommended items and decide on their demand without requiring to consider all items in the system. The provider will decide on well-chosen offerings with correct prices so that it can satisfy the capacity constraints. However, as the provider does not have the complete knowledge of the user preferences, it needs to learn the unknown preference parameters $\vect{\Theta}^*$ from the user feedback so that it can determine better recommendations as well as the correct prices. Based on the examples of applications provided in the introduction, we believe that design of such system dynamics is a practically-relevant question to be resolved.

As a first step in this direction, we decide to restrict our attention to a setting that itself has interesting interactions between learning the user preferences and allocating the items. In order to facilitate our analysis, we consider that the number of choices presented to each user $u$ at round $t$ is limited exactly by the their demand $d_{t, u}$ and users are allocated with all of the items that they are recommended. Therefore, the problem essentially reduces to an allocation problem in which users get allocated a set of items directly by the provider instead of users choosing between the offerings. Then, after each round of allocation, users provide feedback about the items that they have been allocated so that the provider can enhance its performance in the following rounds. Hence, whilst the users get allocated sequentially, the predictions are constantly refined using the reward feedback.

The provider determines the allocations according to an \textit{optimistic} estimate of the true mean reward matrix $\vect{\Theta}^*$. It solves the allocation problem \eqref{integer_num} assuming that the estimated parameter is the underlying reward parameter and obtains an estimate for the optimum allocation at each round $t$. Even though these allocations can be suboptimal due to estimation errors, our analysis shows that the cumulative regret obtained from these sequential allocations can only grow sublinearly with the time horizon. Using this approach, we are coupling the general principle of optimism in the face of uncertainty (OFU) along with capacity aware resource allocation. In the experiments section, we show the importance of this connection by comparing our strategy with algorithms that only focus on one aspect of the problem: a non-OFU algorithm that only aims for achieving momentary performance and an OFU-based algorithm that is unaware of the capacities.

\begin{remark}
\label{remark_price}
When the allocation problem \eqref{integer_num} is solved with the estimated parameters, the Lagrange multipliers for the capacity constraints give estimates for the optimum prices of the items. As long as the user preferences are estimated well enough, these prices emerging from provider's problem are such that users who are aware of their preference for all items would still choose the recommended items. Hence, when the user preferences are learned, the mechanism is able to achieve high-reward allocations that complies with the user incentives under the optimal pricing.
\end{remark}

\subsection{Problem formulation}

In this section, we formulate the provider's problem and its objective. At each time period $t$, the provider chooses multiple user-item allocation pairs collected into a set $\mathcal{A}_t \subseteq [N] \times [M]$. Then, the provider observes a random reward $R_{t, u, i}$ if user $u$ is allocated with item $i$ at round $t$. The total reward is the sum of rewards obtained from the system at all rounds during a time horizon $T$. The task is to repeatedly allocate the items to the users in multiple rounds so that the total expected reward of the system is as close to the reward of the optimal allocation as possible.

Letting $\mathbf{E}_{u, i} \in \mathbb{R}^{N \times M}$ denote the zero-one matrix with a single one at the $(u, i)$ entry, we can write the indicator matrix for the allocation at time $t$ as $\mathbf{X}_{t}  = \sum_{(u, i) \in \mathcal{A}_t} \mathbf{E}_{u, i}$. Consequently, $\mathbf{X}_{t}$ becomes a zero-one matrix with ones at entries $\mathcal{A}_t$ and zeros everywhere else. Note that there is a one-to-one relation between the matrix $\mathbf{X}_{t}$ and the set $\mathcal{A}_t$.

We denote by $H_t$ the history $\{\mathbf{X}_{\tau}, (R_{\tau, u, i})_{(u, i) \in \mathcal{A}_{\tau}}\}_{\tau = 1}^{t-1}$ of observations available to the provider when choosing the next allocation $\mathbf{X}_{t}$. The allocator employs a policy $\pi = \{ \pi_t | t \in \mathbb{N}\}$, which is a sequence of functions, each mapping the history $H_t$ to an action $\mathbf{X}_{t}$. Then, the $T$ period cumulative regret of a policy $\pi$ is the random variable
\vspace{-7pt}
\begin{equation*} \label{cumulative_regret}
    \mathcal{R}(T, \pi) =  \sum_{t = 1}^{T} \left[ \langle \mathbf{X}^*_t, \mathbf{\Theta}^*\rangle - \langle \mathbf{X}_{t}, \mathbf{\Theta}^* \rangle \right]
\end{equation*}
\vspace{-4pt}
where $\vect{X}^*_t \in \argmax_{\vect{X} \in \mathcal{X}_t} \; \langle \vect{X}, \vect{\Theta}^* \rangle$ denotes optimum allocation at time $t$.


\section{Methodology}
\label{sect_methodology}

In order to facilitate our analysis, we start by making the following assumptions that are standard in the multi-armed bandits literature.
\begin{assumption} For all $u \in [N]$, $i \in [M]$ and $t \in \mathbb{N}$, the rewards $R_{t, u, i}$ are independent and $\eta$-sub-Gaussian with mean $\theta^*_{u, i} \in [0, B]$.

\label{rew_assumptio}
\end{assumption}

To model the dependency between the mean rewards obtained from different user-item pairs, we employ the following assumption. We first present our algorithm and theoretical results under the general setting given by this assumption, and specialize for the setting of the collaborative filtering in following sections.
\begin{assumption}
\label{low_assum}
The mean reward matrix $\vect{\Theta}^*$ belongs to a known structure set $\mathcal{L} \subseteq \mathbb{R}^{N \times M}$.
\end{assumption}

In order to make use of initial historical data possibly available to the provider, we assume that the algorithm has access to an initial rough estimate $\overline{\vect{\Theta}}$ that satisfies $\|\overline{\vect{\Theta}} - \vect{\Theta}^* \|_\text{F} \leq G$. Such an estimate can be constructed using an off-the-shelf low-rank matrix completion algorithm on the initialization data. If such observations are not readily available at the time of initialization, they can be obtained by randomly sampling some of the user-item allocation pairs once. It is worth to note that one can also set $\overline{\vect{\Theta}} = \vect{0}$ and let $G$ be some number satisfying $\|\vect{\Theta}^* \|_\text{F} \leq G$.

\vspace{-3pt}
\begin{algorithm}
\caption{Structured Combinatorial Multi-Armed Bandit}
\begin{algorithmic}
\Require horizon $T$, initial estimate $\overline{\vect{\Theta}} \in \mathbb{R}^{N \times M}$ with $\|\overline{\vect{\Theta}} - \vect{\Theta}^* \|_\text{F} \leq G$.
\For{$t = 1, 2, \dots, T$}
\State Find the regularized least squares estimate $ \widehat{\vect{\Theta}}_t = \argmin_{\vect{\Theta} \in \mathcal{L}} \left \{ L_{2,t}(\vect{\Theta}) + \gamma \|\vect{\Theta} - \overline{\vect{\Theta}} \|_2^2 \right \}$
\State Construct the confidence set $\mathcal{C}_t = \{ \vect{\Theta} \in \mathcal{L} : \|\vect{\Theta} - \widehat{\vect{\Theta}}_t \|_{2, E_t} \leq \sqrt{ \beta_t^*(\delta, \alpha, \gamma)}\}$
\State Compute the action vector $\vect{X}_t = \argmax_{\vect{X} \in \mathcal{X}_t} \max_{\vect{\Theta} \in \mathcal{C}_t} \; \langle \vect{X}, \vect{\Theta} \rangle$
\State Play the arms $\mathcal{A}_t$ according to $\vect{X}_t$ 
\State Observe $R_{t, u, i}$ for all $(u, i) \in \mathcal{A}_{t}$
\EndFor
\end{algorithmic}
\label{alg_low}
\end{algorithm}
\vspace{-3pt}

Our method summarized in Algorithm \ref{alg_low} follows the standard OFU (Optimism in Face of Uncertainty) principle \cite{abbasi_2011}. It maintains a confidence set $\mathcal{C}_t$ which contains the true parameter $\vect{\Theta}^*$ with high probability and chooses the allocation $\vect{X}_t$ according to
\begin{equation}
    \vect{X}_t = \argmax_{\vect{X} \in \mathcal{X}_t} \left \{ \max_{\vect{\Theta} \in \mathcal{C}_t} \; \langle \vect{X}, \vect{\Theta} \rangle \right \}
\label{low_oful}
\end{equation}
Typically, the faster the confidence set $\mathcal{C}_t$ shrinks, the lower regret we have. However, the main difficulty is to construct a series of $\mathcal{C}_t$ that leverage the combinatorial observation model as well as the structure of the parameter so that we have low regret bounds. In this work, we consider constructing confidence sets that are centered around the regularized least square estimates. We let the cumulative squared prediciton error at time $t$ be
\begin{equation*}
    L_{2,t}(\vect{\Theta}) = \sum_{\tau=1}^{t-1} \sum_{(u, i) \in \mathcal{A}_\tau} (\theta_{ui}  - R_{\tau, u, i})^2,
\end{equation*}
and define the regularized least squares estimate at time $t$ as
\begin{equation}
    \widehat{\vect{\Theta}}_t = \argmin_{\vect{\Theta} \in \mathcal{L}} \left \{ L_{2,t}(\vect{\Theta}) + \gamma \|\vect{\Theta} - \overline{\vect{\Theta}} \|_2^2 \right \}.
    \label{least_squares_estimate_low}
\end{equation}
Then, the confidence sets take the form $\mathcal{C}_t := \{ \vect{\Theta} \in \mathcal{L} : \|\vect{\Theta} - \widehat{\vect{\Theta}}_t \|_{2, E_t} \leq \sqrt{\beta_t}\}$ where $\beta_t$ is an appropriately chosen confidence parameter, and the regularized empirical 2-norm $\| \cdot \|_{2, E_t}$ is
\begin{equation*}
    \| \vect{\Delta} \|_{2, E_t}^2 := \sum_{u=1}^{N} \sum_{i=1}^{M} (n_{t, u, i} + \gamma) (\Delta_{ui})^2,
\end{equation*}
where $n_{t, u, i} := \sum_{\tau=1}^{t-1} \mathds{1} \{(u,i) \in \mathcal{A}_\tau\}$ is the number of times item $i$ has been allocated to user $u$ before time $t$ (excluding time $t$).
Hence, the empirical 2-norm is a measure of discrepancy that weighs the entries depending on how much they have been explored. Roughly speaking, since the confidence ellipsoid constructed using the 2-norm is wider in directions that are not yet well-explored, the OFU step described in \ref{low_oful} is more inclined to make allocations that include the corresponding user-item pairs. In order to obtain low-regret guarantees for the allocations, the first step is to choose correct $\beta_t$ parameter such that $\mathcal{C}_t$ will contain the true parameter $\vect{\Theta}^*$ for all $t$ with high probability.
In order to take advantage of the structure of the arms, we let $\mathcal{N}(\mathcal{F}, \alpha, \| \cdot \|_{\text{F}})$ denote the $\alpha$-covering number of $\mathcal{F}$ in the Frobenious-norm $\| \cdot \|_{\text{F}}$, and let
\begin{equation*}
    \beta_t^*(\delta, \alpha, \gamma) := 8 \eta^2 \log \left(\mathcal{N}(\mathcal{L}, \alpha, \| \cdot \|_{\text{F}}) / \delta \right) + 2 \alpha t NM \left [ 8 B + \sqrt{8 \eta^2 \log(4NM t^2/\delta)} \right]  + 4 \gamma G^2.
\end{equation*}
Then, the following Lemma establishes that if we set $\beta_t = \beta_t^*(\delta, \alpha, \gamma)$, the resulting confidence sets have the desired properties.
\begin{lemma} For any $\delta > 0$, $\alpha > 0$, $\gamma > 0$, let $\widehat{\vect{\Theta}}_t$ be the regularized least squares estimate given in \ref{least_squares_estimate_low}. If the confidence sets are given as
\begin{equation}
    \mathcal{C}_t := \{ \vect{\Theta} \in \mathcal{L} : \|\vect{\Theta} - \widehat{\vect{\Theta}}_t \|_{2, E_t} \leq \sqrt{ \beta_t^*(\delta, \alpha, \gamma)}\},
    \label{conf_sets_low}
\end{equation}
then with probability at least $1 - 2 \delta$, $\mathcal{C}_t \ni \vect{\Theta}^*$, for all $t \in \mathbb{N}$.
\end{lemma}

Finally, we show that if the structured combinatorial bandits algorithm follows the OFU allocations given in \eqref{low_oful} while constructing the confidence sets according to \eqref{conf_sets_low}, it obtains the following overall regret guarantee:
\begin{theorem}\label{thm_alloc_regret}
Under Assumptions \ref{rew_assumptio} and \ref{low_assum}, for any $\delta > 0$,  $\alpha > 0$, $\gamma \geq 1$, with probability $1 - 2\delta$, the cumulative regret of Algorithm \ref{alg_low} is bounded by
\begin{equation*}
    \mathcal{R}(T, \pi) \leq \sqrt{ 8 N M \beta_T^* (\delta, \alpha, \gamma) T \log \left(1 + T / \gamma \right) }.
\end{equation*}
\end{theorem}

\subsection{Low-Rank COMbinatorial Bandits (LR-COMB)}

\label{sect_lrcb}

As common in collaborative filtering settings, the correlation between users and arms can be captured through a matrix factorization model that leads to a low-rank mean reward matrix. 
Each user $u$ (item $i$) is associated with a feature vector $\vect{p}_u$ ($\vect{q}_i$) in a shared $R$-dimensional space (typically $R \ll M, N$), and the mean reward of each user-item allocation pair is given by $\theta_{ui}^* = \vect{p}_u^\textrm{T} \vect{q}_i$. Consequently, the mean reward matrix satisfies the factorization $\vect{\Theta}^* = \vect{P} \vect{Q}^\textrm{T}$ for some $\vect{P} \in \mathbb{R}^{N \times R}$ and $\vect{Q} \in \mathbb{R}^{M \times R}$. Based on this observation and the boundedness condition given in Assumption \ref{rew_assumptio}, we can choose the structure set $\mathcal{L}$ as
\begin{equation}
    \mathcal{L} = \{ \vect{\Theta} \in \mathds{R}^{N \times M} : \text{rank}(\vect{\Theta}) \leq R, \theta_{ui} \in [0, B], \forall u,i\}.
    \label{low_l}
\end{equation}

Then, Lemma \ref{lemma_covering} in the appendix shows that the covering number for $\mathcal{L}$ given in equation \eqref{low_l} is upper bounded by $\log \mathcal{N}(\mathcal{L}, \alpha, \| \cdot \|_{\text{F}}) \leq (N + M + 1) R \log ( 9B \sqrt{NM} / \alpha )$. Therefore, the regret guarantee for a setting with low-rank mean reward matrix becomes:
\begin{theorem}[Regret of LR-COMB] \label{low_rank_regret_thm}
Under Assumption \ref{rew_assumptio} and Assumption \ref{low_assum} with $\mathcal{L}$ given in \eqref{low_l}, the Algorithm \ref{alg_low} achieves cumulative regret
\begin{equation}
    \mathcal{R}(T, \pi) = \widetilde{\mathcal{O}} \left( \sqrt{N M (N+M) RT} \right),
\end{equation}
\end{theorem}
where $\widetilde{\mathcal{O}}$ is the big-O notation, ignoring the poly-logarithmic factors of $N, M, T, R$. 

In comparison, if we were to ignore the low-rank structure between the mean rewards obtained from user-item allocation pairs and apply the standard combinatorial bandit algorithms (e.g., CUCB \cite{chen_2013}), we would suffer \smash{$\widetilde{\mathcal{O}} ( N M \sqrt{T} )$} regret \cite{kveton_2015}. Since $R \ll M, N$ in many applications of collaborative filtering, our algorithm significantly outperforms this naive approach. As common in the literature of combinatorial bandits, one possible approach to improve upon our theoretical analysis might be by assuming a problem setting where at most $K$ of the arms can be played in each round. However, our current analysis techniques do not allow us to incorporate and leverage such an assumption together with the low-rank structure of collaborative filtering. 

\textbf{Implementation via Matrix Factorization: }
\label{section_mf}
In order to efficiently solve optimization problems  \eqref{low_oful} and \eqref{least_squares_estimate_low} in large scales, we take advantage of the matrix factorization model. As a result, we factorize $\vect{\Theta} = \vect{P} \vect{Q}^\textrm{T}$ where $\vect{P} \in \mathbb{R}^{N \times R}$ and $\vect{Q} \in \mathbb{R}^{M \times R}$, and solve the problems by optimizing over $\vect{P}$ and $\vect{Q}$ rather than directly optimizing over $\vect{\Theta}$. Even if the problem \eqref{least_squares_estimate_low} is not convex in the joint variable ($\vect{P}$, $\vect{Q}$), it is convex in $\vect{P}$ for fixed $\vect{Q}$ and it is convex in $\vect{Q}$ for fixed $\vect{P}$. Therefore, an alternating minimization algorithm becomes a feasible choice to find a reasonable solution for the least squares problem. Similarly, an alternating minimization approach is also useful to solve the problem \eqref{low_oful}. We can fix an allocation $\vect{X}$ and minimize over $\vect{P}$ and $\vect{Q}$. Then, for fixed $\vect{P}$ and $\vect{Q}$, the allocation $\vect{X}$ is determined through the dual decomposition mechanism described in the section \ref{sect_opt_allocations}. We call the resulting algorithm LR-COMB with Matrix Factorization and present it as Algorithm \ref{alg_mf} in the Appendix. 


\section{Experiments}
\label{sect_exp}

In this section, we demonstrate the efficiency of our proposed algorithm by conducting an experimental study over both synthetic and real-world datasets. The goal of our experimental evaluation is twofold: (i) evaluate our algorithm for making online recommendations and allocations in various market settings and (ii) understand the qualitative performance and intuition of our algorithm.

\textbf{Baseline algorithms:} We demonstrate the performance of our method by comparing it with baseline algorithms. To the best of our knowledge, there are no current approaches specifically designed to make interactive recommendations and allocations considering the capacity constraints. Therefore, based on currently available algorithms, we construct our baseline with methods that are designed for similar goals:
\begin{enumerate}[nosep, labelindent= 0pt, align= left, labelsep=0.4em, leftmargin=*]
    \item \textbf{ACF:} (Allocations with Collaborative Filtering) It solves for the least squares problem \eqref{least_squares_estimate_low} to estimate the mean rewards obtained from user-item allocation pairs. Then, makes the best allocation with respect to the estimated parameters at each round.
    \item \textbf{CUCB:} It runs the Combinatorial-UCB algorithm \cite{chen_2013} to decide on allocations without assuming any low-rank structure between the users and items. It views the user-item allocation pairs as arms that have no correlation in between. In every round, it pulls some subset of the arms according to the capacity and demand constraints.
    \item \textbf{ICF:} It runs the Interactive Collaborative Filtering algorithm with linear UCB \cite{zhao_2013} without considering the capacity constraints. For each user, the algorithm recommend the items that it estimates the users will obtain the most reward. Since this method does not consider the capacities, the recommendations do not necessarily satisfy the capacity constraints. Therefore, if an item is recommended to more users than its capacity, we assume that only a randomly chosen subset of the assigned users are able to get the item. The users that are not able to get the item do not send any reward feedback to the system.
    \item \textbf{ICF2:} It is the same as ICF method described above, except that the algorithm observes a zero reward ($R_{t, u, i} = 0$) for all the user-item allocations that were not successfully achieved. As a result of the low rewards obtained from allocations that lead to capacity violations, the algorithm learns to avoid violating the capacities.
\end{enumerate}

\textbf{Experimental setup and datasets:} We use a synthetic dataset and two real world datasets to evaluate our approach. For the synthetic data, we generate an (approximately) low-rank random matrix $\vect{\Theta}^* \in \mathbb{R}^{N \times M}$ with entries from $[0, B]$. For the real-world data, we consider the following publicly available data sets: Movielens 100k \cite{harper_2015} which includes ratings from 943 users on 1682 movies and the RC (Restaurant and Consumer) dataset \cite{blanca_2011} which includes ratings from 138 users on 130 restaurants. As the information of capacities are not given in the considered data, and to the best of our knowledge to any of the publicly available recommendation datasets, we consider instantiating random capacities for all items as described shortly. We consider settings with static and time-varying capacities/demands. For the static case, we assume that all users request one item at all iterations, and the capacity of each each remains unchanged with time. In the dynamic setting, we allow both the demands $\vect{d}_t$ and capacities $\vect{c}_t$ vary with time $t$. At each allocation round, we consider that each entry of $\vect{d}_t$ is independently sampled from a fixed probability distribution over $\{0, 1\}$. Therefore, while active users (with demand 1) are allocated at most one item, the inactive users (with demand 0) do not get allocated any item. Similarly, each entry of $\vect{c}_t$ is independently sampled from a uniform distribution over  $\{0, 1, \dots, C_{\text{max}}\}$. At each round $t$, if user $u$ is allocated the item $i$, the system observes a reward with normal distribution $\mathcal{N}(\theta_{ui}^*, \eta^2)$. 
 
\textbf{Results:} We summarize our results in Figure \ref{fig_regrets}. Further experimental details and results are left to Appendix \ref{sect_additional_exp}. The observations can be summed up into following points: (1) LR-COMB (our proposed approach) is able to achieve lower regret than all other baseline methods in all experimental settings. (2) Even though the ACF method performs slightly better than LR-COMB in the initial rounds, it often gets stuck at high-regret allocations, and hence cannot achieve \emph{no-regret}. It suffers from large regrets in the long-term because it tries to directly exploit the information it acquired so far without making any deliberate explorations. Therefore, we observe the significance of employing a bandit-based approach in achieving a no-regret algorithm. (3) Since CUCB does not leverage the low-rank structure of the parameters, it needs to sample and learn about each user-item allocation pair separately. Hence, it takes much longer for it to learn the optimum allocations. (4) Since ICF does not consider the capacities while making the allocations, it ends up incurring very large regrets. Even if it is able to identify the high-reward allocation pairs via collaborative filtering, the recommendations exceed the respective capacities of the items and we cannot obtain high rewards. (5) One possible ad-hoc approach to mitigate the issues with ICF is to use ICF2 which can indirectly capture the effects of the capacities since it receives zero rewards when the items are not successfully allocated. Nevertheless, ICF2 still does not directly use the knowledge of the capacities and hence it is still quite suboptimal. Even though it is able to show decent performance in static settings, its performance significantly degrades when the capacities dynamically change with time. 

\vspace{-6 pt}
\begin{figure}[ht]
\center
\includegraphics[width=\textwidth]{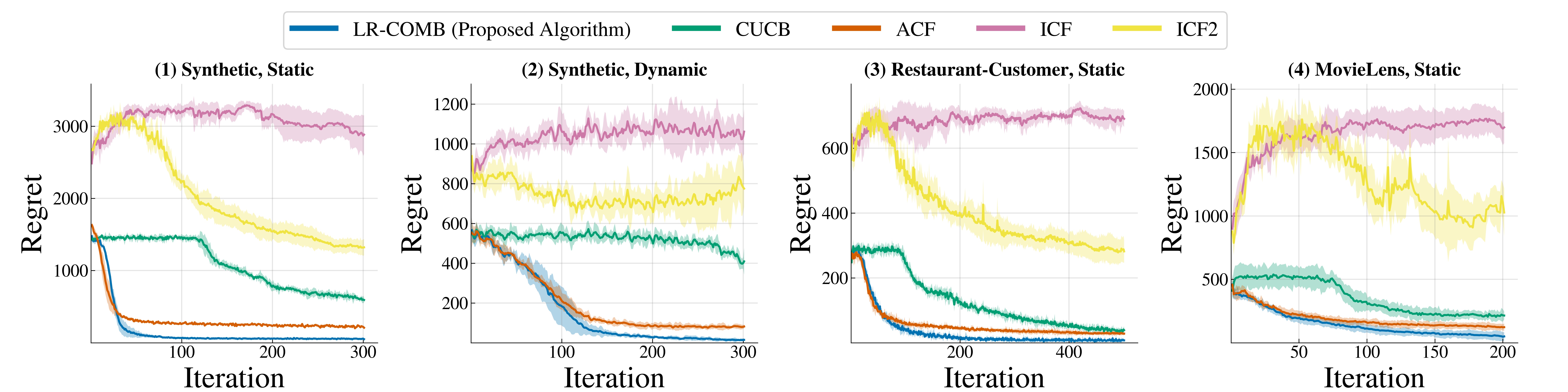}
\caption{Instantaneous regret incurred in each round in different experimental settings. From left to right: (1) synthetic data in a static setting with $N = 800$, $M = 400$, $R = 20$, (2) synthetic data in a dynamic setting with $N = 1000$, $M = 150$, $R = 20$, probability of user activity $0.2$, (3) Restaurant-Customer data in a static setting, (4) Movielens 100k data in a static setting. In all settings, the experiments are run on $10$ problem instances and means are reported together with error regions that indicate one standard deviation of uncertainty.}
\label{fig_regrets}
\end{figure}
\section{Conclusion and future directions}


In this paper, we have studied the setting of interactive
recommendations that achieve socially optimal allocations under capacity constraints. We have formulated the problem as a low-rank combinatorial multi-armed bandit and proposed an algorithm that enjoys low regret. Building on the ideas founded in this work, we aim to pursue joint recommendation and pricing mechanisms that will achieve optimal allocations in the general problem setting with users actively reacting to the recommendations based on the prices determined by the provider. We believe that this is a practically-relevant question to be resolved because it allows for design of many interesting real-world recommendation applications for settings with associated markets.

\clearpage
\bibliographystyle{IEEEtran} 
\bibliography{Bib_Database}

\begin{thebibliography}{10}
\providecommand{\url}[1]{#1}
\csname url@samestyle\endcsname
\providecommand{\newblock}{\relax}
\providecommand{\bibinfo}[2]{#2}
\providecommand{\BIBentrySTDinterwordspacing}{\spaceskip=0pt\relax}
\providecommand{\BIBentryALTinterwordstretchfactor}{4}
\providecommand{\BIBentryALTinterwordspacing}{\spaceskip=\fontdimen2\font plus
\BIBentryALTinterwordstretchfactor\fontdimen3\font minus
  \fontdimen4\font\relax}
\providecommand{\BIBforeignlanguage}[2]{{%
\expandafter\ifx\csname l@#1\endcsname\relax
\typeout{** WARNING: IEEEtran.bst: No hyphenation pattern has been}%
\typeout{** loaded for the language `#1'. Using the pattern for}%
\typeout{** the default language instead.}%
\else
\language=\csname l@#1\endcsname
\fi
#2}}
\providecommand{\BIBdecl}{\relax}
\BIBdecl

\bibitem{schafer_2007}
J.~B. Schafer, D.~Frankowski, J.~Herlocker, and S.~Sen, ``Collaborative
  filtering recommender systems,'' \emph{The Adaptive Web}, vol. 4321, pp.
  291--324, 2007.

\bibitem{bennett_2007}
J.~Bennett and S.~Lanning, ``The netflix prize,'' \emph{Proc. KDD Cup
  Workshop}, pp. 3--6, Aug. 2007.

\bibitem{koren_2009}
Y.~Koren, R.~Bell, and C.~Volinsky, ``Matrix factorization techniques for
  recommender systems,'' \emph{Computer}, vol.~42, no.~8, pp. 30--37, Aug.
  2009.

\bibitem{sarwar_01}
B.~Sarwar, G.~Karypis, J.~Konstan, and J.~Riedl, ``Item-based collaborative
  filtering recommendation algorithms,'' \emph{Proc. 10th International
  Conference on World Wide Web}, pp. 285--295, May 2001.

\bibitem{kawale_2015}
J.~Kawale, H.~H. Bui, B.~Kveton, L.~Tran-Thanh, and S.~Chawla, ``Efficient
  thompson sampling for online matrix-factorization recommendation,''
  \emph{Advances in Neural Information Processing Systems}, vol.~28, pp.
  1297--1305, Dec. 2015.

\bibitem{zhao_2013}
X.~Zhao, W.~Zhang, and J.~Wang, ``Interactive collaborative filtering,''
  \emph{Proceedings of the 22nd ACM International Conference on Information \&
  Knowledge Management}, pp. 1411--1420, Oct. 2013.

\bibitem{smith_1991}
V.~L. Smith, ``Experimental auction markets and the walrasian hypothesis,''
  \emph{Journal of Political Economy}, vol.~73, no.~4, pp. 64--70, 1965.

\bibitem{kelly_97}
F.~Kelly, ``Charging and rate control for elastic traffic,'' \emph{European
  Transactions on Telecommunications}, vol.~8, no.~1, pp. 33--37, Jan. 1997.

\bibitem{barraza_2017}
A.~Barraza-Urbina, ``The exploration-exploitation trade-off in interactive
  recommender systems,'' \emph{Proceedings of the Eleventh ACM Conference on
  Recommender Systems}, pp. 431--435, Aug. 2017.

\bibitem{wang_2019}
Q.~Wang, C.~Zeng, W.~Zhou, T.~Li, S.~S. Iyengar, L.~Shwartz, and G.~Y.
  Grabarnik, ``Online interactive collaborative filtering using multi-armed
  bandit with dependent arms,'' \emph{IEEE Transactions on Knowledge and Data
  Engineering}, vol.~31, no.~8, pp. 1569--1580, Aug. 2019.

\bibitem{dani_2008}
V.~Dani, T.~Hayes, and S.~M. Kakade, ``Stochastic linear optimization under
  bandit feedback,'' \emph{21st Annual Conference on Learning Theory - COLT
  2008, Helsinki, Finland}, pp. 355--366, Jul. 2008.

\bibitem{abbasi_2011}
Y.~Abbasi-Yadkori, ``Improved algorithms for linear stochastic bandits,''
  \emph{Advances in Neural Information Processing Systems}, pp. 2312--2320,
  Dec. 2011.

\bibitem{chen_2013}
W.~Chen, Y.~Wang, and Y.~Yuan, ``Combinatorial multi-armed bandit: General
  framework and applications,'' \emph{Proceedings of the 30th International
  Conference on Machine Learning}, vol.~28, no.~1, pp. 151--159, Jun. 2013.

\bibitem{kveton_2015}
B.~Kveton, Z.~Wen, A.~Ashkan, and C.~Szepesvari, ``{Tight Regret Bounds for
  Stochastic Combinatorial Semi-Bandits},'' \emph{Proceedings of the Eighteenth
  International Conference on Artificial Intelligence and Statistics}, vol.~38,
  pp. 535--543, May 2015.

\bibitem{audibert_2011}
J.-Y. Audibert, S.~Bubeck, and G.~Lugosi, ``Minimax policies for combinatorial
  prediction games,'' \emph{Proceedings of the 24th Annual Conference on
  Learning Theory}, vol.~19, pp. 107--132, Jun. 2011.

\bibitem{johnson_2016}
N.~Johnson, V.~Sivakumar, and A.~Banerjee, ``Structured stochastic linear
  bandits,'' \emph{arXiv preprint arXiv:1606.05693}, 2016.

\bibitem{combes_2017}
R.~Combes, S.~Magureanu, and A.~Proutiere, ``Minimal exploration in structured
  stochastic bandits,'' \emph{Advances in Neural Information Processing
  Systems}, vol.~30, pp. 1761--1769, Dec. 2017.

\bibitem{lu_2021}
Y.~Lu, A.~Meisami, and A.~Tewari, ``Low-rank generalized linear bandit
  problems,'' \emph{Proceedings of the 24th International Conference on
  Artificial Intelligence and Statistics}, vol. 130, pp. 460--468, Apr. 2021.

\bibitem{christakopoulou_2017}
K.~Christakopoulou, J.~Kawale, and A.~Banerjee, ``Recommendation with capacity
  constraints,'' \emph{Proceedings of the 2017 ACM on Conference on Information
  and Knowledge Management}, pp. 1439--1448, Nov. 2017.

\bibitem{makhijani_2019}
R.~Makhijani, S.~Chakrabarti, D.~Struble, and Y.~Liu, ``Lore: A large-scale
  offer recommendation engine with eligibility and capacity constraints,''
  \emph{Proceedings of the 13th ACM Conference on Recommender Systems}, pp.
  160--168, Sep. 2019.

\bibitem{liu_2020}
L.~Liu, H.~Mania, and M.~I. Jordan, ``Competing bandits in matching markets,''
  \emph{Proceedings of the Twenty-Third Conference on Artificial Intelligence
  and Statistics (AISTATS)}, pp. 1618--1628, Aug. 2020.

\bibitem{palomar_2006}
D.~Palomar and M.~Chiang, ``A tutorial on decomposition methods for network
  utility maximization,'' \emph{IEEE Journal on Selected Areas in
  Communications}, vol.~24, no.~8, pp. 1439--1451, Jul. 2006.

\bibitem{harper_2015}
F.~M. Harper and J.~A. Konstan, ``The movielens datasets: History and
  context,'' \emph{ACM Trans. Interact. Intell. Syst.}, vol.~5, no.~4, Dec.
  2015.

\bibitem{blanca_2011}
V.~Blanca, G.~Gabriel, and P.~Rafael, ``Effects of relevant contextual features
  in the performance of a restaurant recommender system,'' \emph{3rd Workshop
  on Context-Aware Recsys}, Oct. 2011.

\bibitem{bertsekas_1991}
D.~P. Bertsekas, \emph{Linear network optimization: algorithms and
  codes}.\hskip 1em plus 0.5em minus 0.4em\relax Mit Press, 1991.

\bibitem{heller_1957}
I.~Heller and C.~B. Tompkins, ``An extension of a theorem of {Dantzig’s},''
  \emph{Linear Inequalities and Related Systems, Annals of Mathematics
  Studies}, vol.~38, pp. 247--254, 1957.

\bibitem{candes_2011}
E.~J. Candes and Y.~Plan, ``Tight oracle inequalities for low-rank matrix
  recovery from a minimal number of noisy random measurements,'' \emph{IEEE
  Transactions on Information Theory}, vol.~57, no.~4, pp. 2342--2359, Mar.
  2011.

\end{thebibliography}


\clearpage
\appendix

\section{Implementation via Matrix Factorization}

The following algorithm describes an efficient implementation of our Low-Rank Combinatorial Bandit algorithm using matrix factorization. Note that converged $\widehat{\vect{\Theta}}_t$ and $\vect{X}_t$ are not necessarily the optimum solution for problems  \eqref{low_oful} and \eqref{least_squares_estimate_low} since the problems are not convex. However, the alternating optimization algorithm guarantees that, in each iteration, the objective value only decreases for \eqref{least_squares_estimate_low}. Similarly, the objective value for \eqref{low_oful} increases in each iteration of the alternating optimization.

\begin{algorithm}
\caption{LR-COMB with Matrix Factorization}
\begin{algorithmic}
\Require horizon $T$, initial estimate $\overline{\vect{\Theta}} \in \mathbb{R}^d$ with $\|\overline{\vect{\Theta}} - \vect{\Theta}^* \|_\text{F} \leq G$, parameters $\delta, \alpha > 0$, $\gamma \geq 1$.
\For{$t = 1, 2, \dots, T$}
\State randomly initialize $\widehat{\vect{P}}$ and $\widehat{\vect{Q}}$
\While{convergence criterion not satisfied}
    \State $\widehat{\vect{P}} \gets \argmin_{\vect{P} \in \mathbb{R}^{N \times R}} \left \{ \sum_{\tau=1}^{t-1} \sum_{(u, i) \in \mathcal{A}_\tau} (\vect{p}_u^\textrm{T} \vect{q}_i  - R_{\tau, u, i})^2 + \gamma \|\vect{P} \vect{Q}^\textrm{T} - \overline{\vect{\Theta}} \|_\text{F}^2 \right \}$
    \State $\widehat{\vect{Q}} \gets \argmin_{\vect{Q} \in \mathbb{R}^{M \times R}} \left \{ \sum_{\tau=1}^{t-1} \sum_{(u, i) \in \mathcal{A}_\tau} (\vect{p}_u^\textrm{T} \vect{q}_i  - R_{\tau, u, i})^2 + \gamma \|\vect{P} \vect{Q}^\textrm{T} - \overline{\vect{\Theta}} \|_\text{F}^2 \right \}$
\EndWhile
\State $\widehat{\vect{\Theta}}_t \gets \widehat{\vect{P}} \widehat{\vect{Q}}^\textrm{T}$
\State $\vect{X} \gets \mathbb{1}_{N \times M}$, $\vect{P} \gets \widehat{\vect{P}}$, $\vect{Q} \gets \widehat{\vect{Q}}$
\While{convergence criterion not satisfied}
\While{convergence criterion not satisfied}
\State $\vect{P} \gets \argmax_{\vect{P} \in \mathbb{R}^{N \times R}} \langle \vect{X}, \vect{P} \vect{Q}^\textrm{T} \rangle$ s.t. $\|\vect{P} \vect{Q}^\textrm{T} - \widehat{\vect{\Theta}}_t \|_{2, E_t} \leq \sqrt{ \beta_t^*(\delta, \alpha, \gamma)}$
\State $\vect{Q} \gets \argmax_{\vect{Q} \in \mathbb{R}^{M \times R}} \langle \vect{X}, \vect{P} \vect{Q}^\textrm{T} \rangle$ s.t. $\|\vect{P} \vect{Q}^\textrm{T} - \widehat{\vect{\Theta}}_t \|_{2, E_t} \leq \sqrt{ \beta_t^*(\delta, \alpha, \gamma)}$
\EndWhile

\State $\vect{\Theta} \gets \vect{P} \vect{Q}^\textrm{T}$

\While{convergence criterion not satisfied}
    \For{$u \in [N]$} 
        \State $\vect{x}_u\gets \argmax_{\vect{x}} \left \{ \vect{x}^\textrm{T} (\vect{\theta}_u - \vect{\lambda}) \middle | \vect{x} \in \{0, 1\}^{M}, \vect{x}^\textrm{T} \mathbb{1}_M \leq d_{t, u} \right \} $
    \EndFor 
    \State $\vect{\lambda} \gets \left[ \vect{\lambda} - \alpha \left( \vect{c}_t - \sum_{u = 1}^{N} \vect{x}_u \right) \right]^{+} $
\EndWhile
\State $\vect{X} \gets [\vect{x}_1, \vect{x}_2, \dots, \vect{x}_N]^\text{T}$
\EndWhile
\State $\vect{X}_t \gets \vect{X}$
\State Play the arms $\mathcal{A}_t$ according to $\vect{X}_t$ 
\State Observe $R_{t, u, i}$ for all $(u, i) \in \mathcal{A}_{t}$
\EndFor
\end{algorithmic}
\label{alg_mf}
\end{algorithm}

\section{Relaxation of Integer Program}
\label{appendix_num}

A traditional linear integer program (IP) in matrix form is formulated as
\begin{align}
\begin{split}
    \max_{\mathbf{x}} &\; \mathbf{t}^\mathrm{T} \mathbf{x}\\
    \text{s.t.} &\; \mathbf{A} \mathbf{x} \leq \mathbf{b}\\
    &\; \mathbf{x} \in \mathds{Z}_+^d\\
    \label{int_prog}
\end{split}
\end{align}

This problem can be relaxed to a linear program by dropping the integral constraints (setting $\mathbf{x} \in \mathds{R}_+^d$). The integrality gap of an integer program is defined as the difference between the optimal values of the integer program in (IP) and its relaxed linear program. When the vector $\mathbf{b}$ is integral and the matrix $\mathbf{A}$ is totally unimodular (all entries are 1, 0, or -1 and every square sub-minor has determinant of +1 or -1) then the integrality gap is zero and the solution of the relaxed linear program is integer valued \cite{bertsekas_1991}. 

Hence, we can solve \eqref{int_prog} by instead solving the following relaxed linear program:
\begin{align}
\begin{split}
    \max_{\mathbf{x}} &\; \mathbf{t}^\mathrm{T} \mathbf{x}\\
    \text{s.t.} &\; \mathbf{A} \mathbf{x} \leq \mathbf{b}\\
    &\; \mathbf{x} \in \mathds{R}_+^d\\
\end{split}
\end{align}

For a matrix $\vect{A}$ whose rows can be partitioned into two disjoint sets $\mathcal{C}$  and $\mathcal{D}$, the following four conditions together are sufficient for $\vect{A}$ to be totally unimodular \cite{heller_1957}:
\begin{enumerate}[nosep, labelindent= 2pt, align= left, labelsep=0.4em, leftmargin=*]
    \item Every entry in $\vect{A}$ is $0$, $+1$, or $-1$.
    \item Every column of $\vect{A}$ contains at most two non-zero entries.
    \item If two non-zero entries in a column of $\vect{A}$ have the same sign, then the row of one is in $\mathcal{C}$  , and the other in $\mathcal{D}$ .
    \item If two non-zero entries in a column of $\vect{A}$ have opposite signs, then the rows of both are in $\mathcal{C}$  , or both in $\mathcal{D}$ .
\end{enumerate}

In the setting of resource allocation, we can write problem \eqref{integer_num} equivalently as problem \eqref{int_prog} where $\mathbf{x} = \text{vec}(\mathbf{X})$, $\mathbf{t} = \text{vec}(\mathbf{\Theta^*})$, $\mathbf{A}$ and $\mathbf{b}$ are given as
\begin{equation}
    \mathbf{A} = 
    \begin{bmatrix}
    \mathds{1}_N^\mathrm{T} \otimes \mathbf{I}_M\\
    \mathbf{I}_N \otimes \mathds{1}_M^\mathrm{T}\\
    \end{bmatrix}
    \qquad 
    \mathbf{b} = 
    \begin{bmatrix}
    \vect{c}_t\\
    \vect{d}_t\\
    \end{bmatrix}
\label{def_Ab}
\end{equation}

For matrix $\vect{A}$ given in \eqref{def_Ab}, we can set $\mathcal{C}$ to be the set of first $M$ rows corresponding to the capacity constraints, and $\mathcal{D}$ to be the set of remaining rows corresponding to the demand constraints. Since this $\vect{A}$ matrix satisfies the conditions of the proposition for sets $\mathcal{C}$  and $\mathcal{D}$, we obtain that $\mathbf{A}$ is totally unimodular. Finally, since the vector $\mathbf{b}$ is integral and the matrix $\mathbf{A}$ is totally unimodular, the integrality gap is zero.

\section{Structured combinatorial multi-armed bandits}
\label{SCMAB}

For the ease of exposition, we present our proofs in the following setting with $d$ structured arms.

We consider a CMAB (Combinatorial Multi Armed Bandit) problem setting with $d$ arms associated with a set of independent random rewards $R_{t, i}$ for $i \in [d]$ and $t \in \mathbb{N}$. Assume that the set of rewards $\{R_{t, i} | t \in \mathbb{N}\}$ associated with arm $i$ are $\eta$-sub-Gaussian with mean $\theta_i^* \in [0, B]$. Let $\vect{\theta}^* = (\theta_1, \theta_2, \dots, \theta_d)$ be the vector of expectations of all arms and assume we know that it belongs to a structure set $\mathcal{F} \subseteq \mathds{R}^{d}$. We further assume that we have access to an initial rough estimate $\overline{\vect{\theta}}$ that satisfies $\|\overline{\vect{\theta}} - \vect{\theta}^* \|_2 \leq G$ (one can also set $\overline{\vect{\theta}} = \vect{0}$ and let $G$ be some number satisfying $\|\vect{\theta}^* \|_2 \leq G$.).

At each round $t$, a subset of arms $\mathcal{A}_t \subseteq [d]$ are played and the individual outcomes of arms in $\mathcal{A}_t$ are revealed.
The total reward at round $t$ is the sum of the rewards obtained from all arms in $\mathcal{A}_t$. 
Letting $\vect{e}_{i} \in \mathbb{R}^{d}$ denote the zero-one vector with a single one at the $i$-th entry, 
define the action vector for time $t$ as $\vect{x}_{t}  = \sum_{i \in \mathcal{A}_t} \vect{e}_{i}$. 
Consequently, $\vect{x}_{t}$ becomes a zero-one vector with ones at entries $\mathcal{A}_t$ and zeros everywhere else. 
The problem contains a constraint that any valid action $\vect{x}_{t}$ must belong to a (time-varying) constraint set $\mathcal{X}_{t} \subseteq \{0, 1\}^{d}$. 

The optimum allocation $\vect{x}^*_t$ at time $t$ is given by $\vect{x}^*_t \in \argmax_{\vect{x} \in \mathcal{X}_t} \; \langle \vect{x}, \vect{\theta}^* \rangle$. We denote by $H_t$ the history $\{\vect{x}_{\tau}, (R_{\tau, i})_{i \in \mathcal{A}_{\tau}}\}_{\tau = 1}^{t-1}$ of observations available when choosing the next action $\vect{x}_{t}$. Let $\pi$ be a policy which takes the action $\vect{x}_{t}$ using the history $H_t$. Then, the $T$ period regret of a policy $\pi$ is the random variable $ \mathcal{R}(T, \pi) =  \sum_{t = 1}^{T} \left[ \langle \vect{x}^*_t, \vect{\theta}^*\rangle - \langle \vect{x}_{t}, \vect{\theta}^* \rangle \right] $.

\subsection{OFU for Structured Combinatorial Bandits}

We present our algorithm in the setting where the mean reward vector $\vect{\theta}^*$ belongs to a structure set $\mathcal{F} \subseteq \mathds{R}^{d}$. Then, we analyze the algorithm to establish performance guarantees.

\begin{algorithm}
\caption{OFU for Structured Combinatorial Bandits}\label{alg:cap}
\begin{algorithmic}
\Require horizon $T$, initial estimate $\overline{\vect{\theta}} \in \mathbb{R}^d$ with $\|\overline{\vect{\theta}} - \vect{\theta}^* \|_2 \leq G$, parameters $\delta, \alpha > 0$, $\gamma \geq 1$.
\For{$t = 1, 2, \dots, T$}
\State Find the least squares estimate $ \widehat{\vect{\theta}}_t = \argmin_{\vect{\theta} \in \mathcal{F}} \left \{ L_{2,t}(\vect{\theta}) + \gamma \|\vect{\theta} - \overline{\vect{\theta}} \|_2^2 \right \}$
\State Construct the confidence set $\mathcal{C}_t = \{ \vect{\theta} \in \mathcal{F} : \|\vect{\theta} - \widehat{\vect{\theta}}_t \|_{2, E_t} \leq \sqrt{ \beta_t^*(\delta, \alpha, \gamma)}\}$
\State Compute the action vector $\vect{x}_t = \argmax_{\vect{x} \in \mathcal{X}_t} \max_{\vect{\theta} \in \mathcal{C}_t} \; \langle \vect{x}, \vect{\theta} \rangle$
\State Play the arms $\mathcal{A}_t$ according to $\vect{x}_t$ 
\State Observe $(R_{\tau, i})_{i \in \mathcal{A}_{\tau}}$
\EndFor
\end{algorithmic}
\label{alg_comb}
\end{algorithm}

The algorithm maintains a confidence set $\mathcal{C}_t$ that contains the true parameter $\vect{\theta}^*$ with high probability and chooses the action $\vect{x}_t$ according to
\begin{equation}
    (\vect{x}_t, \widetilde{\vect{\theta}}_t) = \argmax_{(\vect{x}, \vect{\theta}) \in \mathcal{X}_t \times \mathcal{C}_t} \; \langle \vect{x}, \vect{\theta} \rangle
\end{equation}

The confidence sets that we construct are centered around the regualarized least square estimates defined next. We first let the cumulative squared prediciton error at time $t$ be
\begin{equation}
    L_{2,t}(\vect{\theta}) = \sum_{\tau=1}^{t-1} \sum_{i \in \mathcal{A}_\tau} (\theta_i  - R_{\tau, i})^2
\end{equation}
and define the regularized least squares estimate at time $t$ as
\begin{equation}
    \widehat{\vect{\theta}}_t = \argmin_{\vect{\theta} \in \mathcal{F}} \left \{ L_{2,t}(\vect{\theta}) + \gamma \|\vect{\theta} - \overline{\vect{\theta}} \|_2^2 \right \}
\end{equation}

Then, the confidence sets take the form $\mathcal{C}_t := \{ \vect{\theta} \in \mathcal{F} : \|\vect{\theta} - \widehat{\vect{\theta}}_t \|_{2, E_t} \leq \sqrt{\beta_t}\}$ where $\beta_t$ is an appropriately chosen confidence parameter, and the regularized empirical 2-norm $\| \cdot \|_{2, E_t}$ is defined by 
\begin{equation*}
    \| \vect{\Delta} \|_{2, E_t}^2 := \sum_{\tau=1}^{t-1} \sum_{i \in \mathcal{A}_\tau} \langle \vect{\Delta}, \vect{e}_{i} \rangle^2 + \gamma \| \vect{\Delta} \|_{2}^2 =  \sum_{i=1}^{d} (n_{t, i} + \gamma) (\Delta_{i})^2
\end{equation*}
where $n_{t, i} := \sum_{\tau=1}^{t-1} \mathds{1} \{i \in \mathcal{A}_t\}$ denotes the number of times arm $i$ has been pulled before time $t$ (excluding time $t$). For future reference, we also define the (non-regularized) empirical 2-norm $\| \cdot \|_{2, \widetilde{E}_t}$ by 
\begin{equation*}
    \| \vect{\Delta} \|^2_{2, \widetilde{E}_t} := \sum_{\tau=1}^{t-1} \sum_{i \in \mathcal{A}_\tau} \langle \vect{\Delta}, \vect{e}_{i} \rangle^2 =  \sum_{i=1}^{d} (n_{t, i}) (\Delta_{i})^2
\end{equation*}

Note that the regularized empirical 2-norm is related to (non-regularized) empirical 2-norm as
\begin{equation*}
    \| \vect{\Delta} \|^2_{2, E_t^2} = \| \vect{\Delta} \|^2_{2, \widetilde{E}_t} + \gamma \| \vect{\Delta} \|^2_{2}
\end{equation*}

By Lemma \ref{lemma_lower_bound}, we establish that for any $\vect{\theta} \in \mathbb{R}^{d}$, 
\begin{equation}
    \mathds{P} \left( L_{2,t}(\vect{\theta}) \geq L_{2,t}(\vect{\theta}^*) + \frac{1}{2} \|\vect{\theta}^* - \vect{\theta} \|^2_{2, \widetilde{E}_t} - 4 \eta^2 \log(1/\delta) \quad ,\forall t \in \mathbb{N} \right) \geq 1 - \delta
\end{equation}

Hence, with high probability, $\vect{\theta}$ can achieve lower squared error than $\vect{\theta}^*$ only if the empirical deviation $\|\vect{\theta}^* - \vect{\theta} \|^2_{2, \widetilde{E}_t}$ is less than $8 \eta^2 \log(1/\delta)$. 

In order to make this property hold uniformly for all $\vect{\theta}$ in a subset $\mathcal{C}_t$ of $\mathcal{F}$, we discretize $\mathcal{C}_t$ at some discretization scale $\alpha$ and apply a union bound for this finite discretization set. Let $\mathcal{N}(\mathcal{F}, \alpha, \| \cdot \|_{2})$ denote the $\alpha$-covering number of $\mathcal{F}$ in the 2-norm $\| \cdot \|_{2}$, and let 
\begin{equation}
    \beta_t^*(\delta, \alpha, \gamma) := 8 \eta^2 \log \left(\mathcal{N}(\mathcal{F}, \alpha, \| \cdot \|_{2}) / \delta \right) + 2 \alpha t \sqrt{d} \left [ 8 B + \sqrt{8 \eta^2 \log(4d t^2/\delta)} \right]  + 4 \gamma G^2
\end{equation}

Then, Lemma \ref{lemma_confidence} shows that if we set $\beta_t = \beta_t^*(\delta, \alpha)$, the confidence sets $\mathcal{C}_t$ contain the true parameter $\vect{\theta}^*$ for all $t$ with high probability. Following the construction of the confidence sets, the next step is to obtain the overall regret guarantee. As given in Corollary \ref{corr_regret_order}, we find that the regret of Algorithm \ref{alg_comb} satisfies
\begin{equation}
    \mathcal{R}(T, \pi) = \widetilde{\mathcal{O}} \left( \sqrt{\eta^2 d \log \left(\mathcal{N}(\mathcal{F}, T^{-1}, \| \cdot \|_{2})\right) T} \right)
\end{equation}

\begin{remark}
In the literature of linear bandits, the typical observation model is such that each action $\vect{x}_{t}$ results in a single reward feedback with mean $\langle \vect{x}_{t}, \vect{\theta}^* \rangle$ and sub-Gaussianity parameter $d \eta^2$ (since each arm has a $\eta^2$-sub-Gaussian reward). Therefore, that observation model can only obtain $\widetilde{\mathcal{O}} ( \sqrt{\eta^2 d^2 \log \left(\mathcal{N}(\mathcal{F}, T^{-1}, \| \cdot \|_{2})\right) T} )$ regret guarantee. However, in our setting, the observations are sets of independent rewards $\{R_{t, i}\}_{i \in \mathcal{A}_t}$ where each element is $\eta^2$-sub-Gaussian. Due to this richer nature of the observation model, we are able to achieve lower regret guarantees than only observing a single cumulative reward.
\end{remark}

\section{Proofs for Confidence Sets}
\label{pf_conf_sets}

\subsection{Martingale Exponential Inequalities}

We start with preliminary results on martingale exponential inequalities.

Consider a sequence of random variables $(Z_n)_{n \in \mathds{N}}$ adapted to the filtration $(\mathcal{H}_n)_{n \in \mathds{N}}$. Assume $\mathds{E}[\exp (\lambda Z_i)]$ is finite for all $\lambda$. Define the conditional mean $\mu_i = \mathds{E}[Z_i | \mathcal{H}_{i-1}]$, and define the conditional cumulant generating function of the centered random variable $[Z_i - \mu_i]$ by $\psi_i(\lambda) := \log \mathds{E}[ \exp (\lambda [Z_i - \mu_i]) | \mathcal{H}_{i-1}]$. Let 
\begin{equation*}
    M_n(\lambda) = \exp \left \{ \sum_{i=1}^{n} \lambda [Z_i - \mu_i] - \psi_i(\lambda) \right \} 
\end{equation*}

\begin{lemma}
$(M_n (\lambda))_{n \in \mathds{N}}$ is a martingale with respect to the filtration $(\mathcal{H}_n)_{n \in \mathds{N}}$, and $\mathds{E}[M_n (\lambda)] = 1$.
\label{martingale}
\end{lemma}

\begin{proof}
By definition, we have
\begin{equation*}
    \mathds{E}[M_1 (\lambda) | \mathcal{H}_0] = \mathds{E}[\exp \{ \lambda [Z_1 - \mu_1] - \psi_1(\lambda) \} | \mathcal{H}_0] = 1 
\end{equation*}
Then, for any $n \geq 2$,
\begin{align*}
    \mathds{E}[M_n (\lambda) | \mathcal{H}_{n-1}] &= \mathds{E}[M_{n-1}(\lambda) \exp \{ \lambda [Z_n - \mu_n] - \psi_n(\lambda) \} | \mathcal{H}_{n-1}] \\
    &= M_{n-1}(\lambda) \mathds{E}[\exp \{ \lambda [Z_n - \mu_n] - \psi_n(\lambda) \} | \mathcal{H}_{n-1}]\\
    &= M_{n-1}(\lambda)\\
\end{align*}
since $M_{n-1}(\lambda)$ is a measurable function of the filtration $\mathcal{H}_{n-1}$.
\end{proof}

\begin{lemma}
For all $x \geq 0$ and $\lambda \geq 0$, 
\begin{equation*}
\mathds{P} \left(\sum_{i=1}^{n} \lambda Z_i \leq x + \sum_{i=1}^{n} [\lambda \mu_i + \psi_i(\lambda)] \quad ,\forall t \in \mathds{N} \right) \geq 1 - e^{-x}
\end{equation*}
\label{exp_martingale}
\end{lemma}

\begin{proof}
For any $\lambda$, $(M_n (\lambda))_{n \in \mathds{N}}$ is a martingale with respect to $(\mathcal{H}_n)_{n \in \mathds{N}}$ and $\mathds{E}[M_n (\lambda)] = 1$ by Lemma \ref{martingale}. For arbitrary $x \geq 0$, define $\tau_x = \inf \{ n\geq 0 | M_n(\lambda) \geq x \}$ and note that $\tau_x$ is a stopping time corresponding to the first time $M_n$ crosses the boundary $x$. Since $\tau$ is a stopping time with respect to $(\mathcal{H}_n)_{n \in \mathds{N}}$, we have $\mathds{E}[M_{\tau_x \wedge n}(\lambda)] = 1$. Then, by Markov's inequality
\begin{equation*}
    x \mathds{P} (M_{\tau_x \wedge n}(\lambda) \geq x) \leq \mathds{E}[M_{\tau_x \wedge n}(\lambda)] = 1
\end{equation*}

Noting that the event $\{ M_{\tau_x \wedge n}(\lambda) \geq x \} = \bigcup_{k=1}^n \{M_{k}(\lambda) \geq x\} $, we have
\begin{equation*}
\mathds{P} \left( \bigcup_{k=1}^n \{M_{k}(\lambda) \geq x\} \right) \leq \frac{1}{x}
\end{equation*}

Taking the limit as $n \to \infty$, and applying monotone convergence theorem shows that $\mathds{P} \left( \bigcup_{k=1}^\infty \{M_{k}(\lambda) \geq x\} \right) \leq \frac{1}{x}$ or $\mathds{P} \left( \bigcup_{k=1}^\infty \{M_{k}(\lambda) \geq e^x\} \right) \leq e^{-x}$. Then, by definition of $M_k (\lambda)$, we conclude
\begin{equation*}
    \mathds{P} \left( \bigcup_{k=1}^\infty \left \{ \sum_{i=1}^{n} \lambda [Z_i - \mu_i] - \psi_i(\lambda) \geq x \right \}  \right) \leq e^{-x}
\end{equation*}

\end{proof}

\subsection{Proofs for the construction of confidence sets}

\begin{lemma} For any $\delta > 0$ and $\vect{\theta} \in \mathbb{R}^{d}$,
\begin{equation}
    \mathds{P} \left( L_{2,t}(\vect{\theta}) \geq L_{2,t}(\vect{\theta}^*) + \frac{1}{2} \|\vect{\theta}^* - \vect{\theta} \|^2_{2, \widetilde{E}_t} - 4 \eta^2 \log(1/\delta) \quad ,\forall t \in \mathbb{N} \right) \geq 1 - \delta
\end{equation}

\label{lemma_lower_bound}
\end{lemma}

\begin{proof}

Let $\mathcal{H}_{t-1}$ be the $\sigma$-algebra generated by $(H_t, \mathcal{A}_t)$ and let $\mathcal{H}_0 = \sigma(\emptyset, \Omega)$. Then, define $\epsilon_{t, i} := R_{t, i} - \langle \mathbf{X}_{t, i}, \vect{\theta}^*\rangle$ for all $t \in \mathds{N}$ and $i \in \mathcal{A}_t$. By previous assumptions, $\mathds{E} [\epsilon_{t, i} | \mathcal{H}_{t-1}] = 0$ and $\mathds{E} [\exp (\lambda \epsilon_{t, u}) | \mathcal{H}_{t-1} ] \leq \exp \left(\frac{\lambda^2 \eta^2}{2} \right)$ for all $t$.

Define $Z_{t, i} = (R_{t, i} - \langle \mathbf{X}_{t, i}, \vect{\theta}^*\rangle)^2 - (R_{t, i} - \langle \mathbf{X}_{t, i}, \vect{\theta}\rangle)^2$. Then, we have
\begin{align*}
Z_{t, i} &= - (\langle \mathbf{X}_{t, i}, \vect{\theta}\rangle - \langle \mathbf{X}_{t, i}, \vect{\theta}^*\rangle)^2 + 2 \epsilon_{t, i} (\langle \mathbf{X}_{t, i}, \vect{\theta}\rangle - \langle \mathbf{X}_{t, i}, \vect{\theta}^*\rangle)\\
&= - \langle \mathbf{X}_{t, i}, \vect{\theta} - \vect{\theta}^*\rangle^2 + 2 \epsilon_{t, i} \langle \mathbf{X}_{t, i}, \vect{\theta} - \vect{\theta}^*\rangle
\end{align*}

Therefore, the conditional mean and conditional cumulant generating function satisfy
\begin{align*}
    \mu_{t, i} &:= \mathds{E}[Z_{t, i} | \mathcal{H}_{t-1}] = - \langle \mathbf{X}_{t, i}, \vect{\theta} - \vect{\theta}^*\rangle^2\\
    \psi_t(\lambda) &:= \log \mathds{E}[ \exp (\lambda [Z_{t, i} - \mu_{t, i}]) | \mathcal{H}_{t-1}]\\
    &= \log \mathds{E}[ \exp (2 \lambda \langle \mathbf{X}_{t, i}, \vect{\theta} - \vect{\theta}^*\rangle \epsilon_{t, i} ) | \mathcal{H}_{t-1}]\\
    &\leq \frac{(2 \lambda \langle \mathbf{X}_{t, i}, \vect{\theta} - \vect{\theta}^*\rangle)^2 \eta^2}{2}
\end{align*}

Applying Lemma \ref{exp_martingale} shows that for all $x \geq 0$ and $\lambda \geq 0$,
\begin{equation*}
\mathds{P} \left(\sum_{\tau=1}^{t-1} \sum_{u \in \mathcal{A}_\tau} Z_{\tau, i} \leq \frac{x}{ \lambda} + \sum_{\tau=1}^{t-1} \sum_{u \in \mathcal{A}_\tau} \langle \mathbf{X}_{\tau, i}, \vect{\theta} - \vect{\theta}^*\rangle^2 (2 \lambda \eta^2 - 1) \quad ,\forall t \in \mathds{N} \right) \geq 1 - e^{-x}
\end{equation*}

Note that we have $\sum_{\tau=1}^{t-1} \sum_{u \in \mathcal{A}_\tau} Z_{\tau, i} = L_{2,t}(\vect{\theta}^*) - L_{2,t}(\vect{\theta})$,

and $\sum_{\tau=1}^{t-1} \sum_{u \in \mathcal{A}_\tau} \langle \mathbf{X}_{\tau, i}, \vect{\theta} - \vect{\theta}^*\rangle^2 = \|\vect{\theta}^* - \vect{\theta} \|^2_{2, \widetilde{E}_t}$. 

Then, choosing $\lambda = \frac{1}{4 \eta^2}$ and $x = \log \frac{1}{\delta}$ gives
\begin{equation*}
    \mathds{P} \left( L_{2,t}(\vect{\theta}) \geq L_{2,t}(\vect{\theta}^*) + \frac{1}{2} \|\vect{\theta}^* - \vect{\theta} \|^2_{2, \widetilde{E}_t} - 4 \eta^2 \log(1/\delta) \quad ,\forall t \in \mathbb{N} \right) \geq 1 - \delta
\end{equation*}

\end{proof}

\begin{lemma}
If $\vect{\theta}^{\alpha} \in \mathcal{F}^{\alpha}$ satisfies $\|\vect{\theta} - \vect{\theta}^{\alpha}\|_{2} \leq \alpha$, then with probability at least $1 - \delta$,
\begin{equation}
    \left |  \frac{1}{2} \|\vect{\theta}^* - \vect{\theta}^{\alpha} \|^2_{2, \widetilde{E}_t} - \frac{1}{2} \|\vect{\theta}^* - \vect{\theta} \|^2_{2, \widetilde{E}_t} + L_{2,t}(\vect{\theta}) - L_{2,t}(\vect{\theta}^{\alpha}) \right | \leq \alpha t d \left [ 8 B + \sqrt{8 \eta^2 \log(4d t^2/\delta)} \right]
\label{eqn_of_discr_lemma}
\end{equation}
\label{discr_lemma}
\end{lemma}

\begin{proof}

Since any two $\vect{\theta}, \vect{\theta}^{\alpha} \in \mathcal{F}$ satisfy $\|\vect{\theta} - \vect{\theta}^{\alpha}\|_{2} \leq \sqrt{d} B$, it is enough to consider $\alpha \leq  \sqrt{d} B$. We find 
\begin{align*}
    \sum_{ i = 1 }^{d} | \langle \vect{\theta}, \vect{e}_{i} \rangle^2 - \langle \vect{\theta}^{\alpha}, \vect{e}_{i} \rangle^2 | &\leq \max_{ \|\vect{\Delta}\|_2 \leq \alpha } \left \{\sum_{ i = 1 }^{d} \left | \theta_{i}^2 - (\theta_{i}+\Delta_i)^2  \right | \right \}\\
    &= \max_{ \|\vect{\Delta}\|_2 \leq \alpha } \left \{\sum_{ i = 1 }^{d}\left | 2 \theta_{i} \Delta_i + \Delta_i^2  \right | \right \}\\
    &\leq \max_{ \|\vect{\Delta}\|_2 \leq \alpha } \left \{ 2 \sum_{ i = 1 }^{d}\left | \theta_{i} \Delta_i \right | + \sum_{ i = 1 }^{d} \Delta_i^2 \right \}\\
    &\leq \max_{ \|\vect{\Delta}\|_2 \leq \alpha }  \left \{ 2 B \|\vect{\Delta}\|_1 + \|\vect{\Delta}\|_2^2 \right \}\\
    &\leq 2 B \sqrt{d} \alpha + \alpha^2
\end{align*}
Therefore, it implies
\begin{align*}
     \sum_{ i = 1 }^{d} | \langle \vect{\theta} - \vect{\theta}^*, \vect{e}_{i} \rangle^2 - \langle \vect{\theta}^{\alpha} - \vect{\theta}^*, \vect{e}_{i} \rangle^2 | &= \sum_{ i = 1 }^{d} \left| \langle \vect{\theta}, \vect{e}_{i} \rangle^2 - \langle \vect{\theta}^{\alpha}, \vect{e}_{i} \rangle^2 + 2 \langle \vect{\theta}^*, \vect{e}_{i} \rangle \langle \vect{\theta}^{\alpha} - \vect{\theta}, \vect{e}_{i} \rangle \right| \\
     &\leq 2 B \sqrt{d} \alpha + \alpha^2 + 2 B \|\vect{\theta} - \vect{\theta}^{\alpha}\|_{1} \\
     &\leq 4 B \sqrt{d} \alpha + \alpha^2
\end{align*}

Similarly, for any $t$, we have
\begin{align*}
     \sum_{ i = 1 }^{d} | \left( R_{t,i} - \langle \vect{\theta}, \vect{e}_{i} \rangle \right)^2 -\left( R_{t,i} - \langle \vect{\theta}^{\alpha}, \vect{e}_{i} \rangle \right)^2 | &= \sum_{ i = 1 }^{d} \left| 2 R_{t,i} \langle \vect{\theta}^{\alpha} - \vect{\theta}, \vect{e}_{i} \rangle + \langle \vect{\theta}, \vect{e}_{i} \rangle^2 - \langle \vect{\theta}^{\alpha}, \vect{e}_{i} \rangle^2 \right| \\
      &\leq 2 \sum_{ i = 1 }^{d} \left| R_{t,i} \right | \left| \langle \vect{\theta}^{\alpha} - \vect{\theta}, \vect{e}_{i} \rangle \right| + 2 B \sqrt{d} \alpha + \alpha^2\\
     &\leq 2 \| \vect{\theta}^{\alpha} - \vect{\theta}\|_{\infty} \sum_{ i = 1 }^{d} | R_{t,i} | + 2 B \sqrt{d} \alpha + \alpha^2\\
     &\leq 2 \alpha \sum_{ i = 1 }^{d} | R_{t,i} | + 2 B \sqrt{d} \alpha + \alpha^2
\end{align*}

Summing over $t$ and noting that $\mathcal{A}_t \subseteq [d]$, the left hand side of \eqref{eqn_of_discr_lemma} is bounded by
\begin{equation*}
    \sum_{\tau = 1}^{t-1} \left( \frac{1}{2} \left[ 4 B \sqrt{d} \alpha + \alpha^2 \right] + 2 \alpha \sum_{ i = 1 }^{d} | R_{t,i} | + 2 B \sqrt{d} \alpha + \alpha^2  \right) \leq \alpha \sum_{\tau = 1}^{t-1} \left(6 B \sqrt{d} + 2 \sum_{ i = 1 }^{d} | R_{\tau,i} | \right)
\end{equation*}

Because $\epsilon_{\tau, i}$ is $\eta$-sub-Gaussian, $\mathds{P} \left( |\epsilon_{\tau, i}| > \sqrt{2 \eta^2 \log(2/\delta) }\right) \leq \delta$. By a union bound, $\mathds{P} \left( \exists \tau, i \text{ s.t. } |\epsilon_{\tau, i}| > \sqrt{2 \eta^2 \log(4d \tau^2/\delta) }\right) \leq \frac{\delta d}{2} \sum_{\tau = 1}^{\infty} \frac{1}{d \tau^2} \leq \delta$. Since $| R_{\tau,i} | \leq B + | \epsilon_{\tau, i} |$, we have $| R_{\tau,i} | \leq B + \sqrt{2 \eta^2 \log(4d \tau^2/\delta)}$  with probability at least $1 - \delta$. Consequently, the bound for the discretization error becomes
\begin{equation*}
\alpha t d \left [ 8 B + 2 \sqrt{2 \eta^2 \log(4d t^2/\delta)} \right]
\end{equation*}

\end{proof}

\begin{lemma} For any $\delta > 0$, $\alpha > 0$ and $\gamma > 0$, if
\begin{equation}
    \mathcal{C}_t = \{ \vect{\theta} \in \mathcal{F} : \|\vect{\theta} - \widehat{\vect{\theta}}_t \|_{2, E_t} \leq \sqrt{ \beta_t^*(\delta, \alpha, \gamma)}\}
\end{equation}
for all $t \in \mathbb{N}$, then
\begin{equation}
    \mathds{P} \left( \vect{\theta}^* \in \mathcal{C}_t \quad ,\forall t \in \mathbb{N} \right) \geq 1 - 2\delta
\end{equation}
\label{lemma_confidence}
\end{lemma}

\begin{proof}

Let $\mathcal{F}^{\alpha} \subset \mathcal{F}$ be an $\alpha$-cover of $\mathcal{F}$ in the 2-norm so that for any $\vect{\theta} \in \mathcal{F}$, there exists $\vect{\theta}^{\alpha} \in \mathcal{F}^{\alpha}$ such that $\|\vect{\theta} - \vect{\theta}^{\alpha}\|_{2} \leq \alpha$. By a union bound applied to Lemma \ref{lemma_lower_bound}, with probability at least $1 - \delta$,
\begin{equation*}
    L_{2,t}(\vect{\theta}^{\alpha}) - L_{2,t}(\vect{\theta}^*) \geq \frac{1}{2} \|\vect{\theta}^* - \vect{\theta}^{\alpha} \|^2_{2, \widetilde{E}_t} - 4 \eta^2 \log(|\mathcal{F}^{\alpha}|/\delta) \quad ,\forall \vect{\theta}^{\alpha} \in \mathcal{F}^{\alpha}, t \in \mathbb{N}
\end{equation*}

Therefore, with probability at least $1 - \delta$, for all $\vect{\theta} \in \mathcal{F}, t \in \mathbb{N}$,
\begin{align*}
    L_{2,t}(\vect{\theta}) - L_{2,t}(\vect{\theta}^*) \geq & \frac{1}{2} \|\vect{\theta}^* - \vect{\theta} \|^2_{2, \widetilde{E}_t} - 4 \eta^2 \log(|\mathcal{F}^{\alpha}|/\delta) \\
    &+ \min_{\vect{\theta}^{\alpha} \in \mathcal{F}^{\alpha}} \left \{  \frac{1}{2} \|\vect{\theta}^* - \vect{\theta}^{\alpha} \|^2_{2, \widetilde{E}_t} - \frac{1}{2} \|\vect{\theta}^* - \vect{\theta} \|^2_{2, \widetilde{E}_t} + L_{2,t}(\vect{\theta}) - L_{2,t}(\vect{\theta}^{\alpha}) \right \}
\end{align*}

By Lemma \ref{discr_lemma}, with probability at least $1 - 2 \delta$,
\begin{equation*}
    L_{2,t}(\vect{\theta}) - L_{2,t}(\vect{\theta}^*) \geq \frac{1}{2} \|\vect{\theta}^* - \vect{\theta} \|^2_{2, \widetilde{E}_t} - D_t
\end{equation*}
where $D_t := 4 \eta^2 \log(|\mathcal{F}^{\alpha}|/\delta) + \alpha t d \left [ 8 B + \sqrt{8 \eta^2 \log(4d t^2/\delta)} \right]$.

Adding the regularization terms to both sides, we obtain
\begin{equation*}
    L_{2,t}(\vect{\theta}) + \gamma \|\vect{\theta} - \overline{\vect{\theta}} \|_2^2 - L_{2,t}(\vect{\theta}^*) - \gamma \|\vect{\theta}^* - \overline{\vect{\theta}} \|_2^2 \geq \frac{1}{2} \|\vect{\theta}^* - \vect{\theta} \|^2_{2, \widetilde{E}_t} + \gamma \|\vect{\theta} - \overline{\vect{\theta}} \|_2^2 - D_t - \gamma  \|\vect{\theta}^* - \overline{\vect{\theta}} \|_2^2 
\end{equation*}

Note the definition of the least square estimate $\widehat{\vect{\theta}}_t = \argmin_{\vect{\theta} \in \mathcal{F}} \left \{ L_{2,t}(\vect{\theta}) + \gamma \|\vect{\theta} - \overline{\vect{\theta}} \|_2^2 \right \}$. By letting $\vect{\theta} = \widehat{\vect{\theta}}_t$, the left hand side becomes non-positive, and hence 
\begin{equation*}
\frac{1}{2} \|\vect{\theta}^* - \widehat{\vect{\theta}}_t \|^2_{2, \widetilde{E}_t} \leq D_t + \gamma \left( \|\vect{\theta}^* - \overline{\vect{\theta}} \|_2^2 - \|\widehat{\vect{\theta}}_t - \overline{\vect{\theta}} \|_2^2 \right)
\end{equation*}

Then, 
\begin{equation*}
\frac{1}{2} \|\vect{\theta}^* - \widehat{\vect{\theta}}_t \|^2_{2, \widetilde{E}_t} + \gamma \left( \|\widehat{\vect{\theta}}_t - \overline{\vect{\theta}} \|_2^2 + \|\vect{\theta}^* - \overline{\vect{\theta}} \|_2^2  \right) \leq D_t + 2 \gamma \|\vect{\theta}^* - \overline{\vect{\theta}} \|_2^2
\end{equation*}

By triangle inequality we have $ \|\widehat{\vect{\theta}}_t - \overline{\vect{\theta}} \|_2 + \|\vect{\theta}^* - \overline{\vect{\theta}} \|_2  \geq \|\vect{\theta}^* - \widehat{\vect{\theta}}_t \|_2$. Taking squares on both sides, we obtain $ \|\widehat{\vect{\theta}}_t - \overline{\vect{\theta}} \|_2^2 + \|\vect{\theta}^* - \overline{\vect{\theta}} \|_2^2  \geq \frac{1}{2} \|\vect{\theta}^* - \widehat{\vect{\theta}}_t \|_2^2$. Then, noting that $\| \vect{\Delta} \|^2_{2, E_t^2} = \| \vect{\Delta} \|^2_{2, \widetilde{E}_t} + \gamma \| \vect{\Delta} \|^2_{2}$, we have
\begin{equation*}
\frac{1}{2} \|\vect{\theta}^* - \widehat{\vect{\theta}}_t \|^2_{2, E_t} \leq D_t + 2 \gamma \|\vect{\theta}^* - \overline{\vect{\theta}} \|_2^2
\end{equation*}

Lastly, using the inequality $\|\vect{\theta}^* - \overline{\vect{\theta}} \|_2^2 \leq G^2$,
\begin{equation*}
\|\vect{\theta}^* - \widehat{\vect{\theta}}_t \|^2_{2, E_t} \leq 8 \eta^2 \log(|\mathcal{F}^{\alpha}|/\delta) + 2 \alpha t d \left [ 8 B + \sqrt{8 \eta^2 \log(4d t^2/\delta)} \right]  + 4 \gamma G^2
\end{equation*}

Taking the infimum over the size of $\alpha$-covers, we obtain the final result.

\end{proof}

\section{Proofs for Regret Bounds}
\label{pf_regrets}

Throughout this section we will use the shorthand $\beta_t =  \beta_t^*(\delta, \alpha, \gamma)$.

We start with the definitions of weighted inner product and norms.

\begin{definition}
For a symmetric positive definite matrix $\vect{W} \in \mathbb{R}^{d \times d}$, define
\begin{itemize}
    \item $\vect{W}$-inner product of two vectors $\vect{x}, \vect{y} \in \mathbb{R}^{d}$ as $\langle \vect{x}, \vect{y} \rangle_{\vect{W}} := \langle \vect{W}  \vect{x}, \vect{y} \rangle$
    \item $\vect{W}$-norm of a vector $\vect{x} \in \mathbb{R}^{d}$ as $\|\vect{x}\|_{\vect{W}} := \sqrt{\langle \vect{x}, \vect{x} \rangle_{\vect{W}}}$.
\end{itemize}
\end{definition}

Then, the regularized empirical 2-norm of a vector $\vect{z} \in \mathbb{R}^d$ can be written as 
\begin{equation}
    \| \vect{z} \|_{2, E_t} = \| \vect{z} \|_{\vect{A}_t}
\end{equation}
where $\vect{A}_t$ is a diagonal matrix with diagonal entries $\{ (n_{t, 1} + \gamma), (n_{t, 2} + \gamma), \dots, (n_{t, d} + \gamma) \}$. 

Recall that $n_{t, i} = \sum_{\tau=1}^{t-1} \mathds{1} \{i \in \mathcal{A}_t\}$ denotes the number of times arm $i$ has been pulled before time $t$ (excluding time $t$).

\begin{lemma}
Let $x \in \mathcal{X}$ and $\Theta \in \mathcal{F}_t$. Then,
\begin{equation}
    |\langle \vect{\theta} - \widehat{\vect{\theta}}_t^{\text{LS}}, \vect{x} \rangle| \leq w \sqrt{\beta_t}
\end{equation}
where $w = \| \vect{x} \|_{\vect{A}_t^{-1}}$ is the "confidence width" of an action $\vect{x}$ at time $t$.
\label{lemma_conf_width}
\end{lemma}

\begin{proof} Let $\vect{\Delta} = \vect{\theta} - \widehat{\vect{\theta}}_t^{\text{LS}}$. Then,
\begin{align*}
    |\langle \vect{\Delta} , \vect{x} \rangle| &= |\vect{\Delta} ^\textrm{T} \vect{x} |\\
    &= |\vect{\Delta}^\textrm{T} \vect{A}_t^{1/2} \vect{A}_t^{-1/2} \vect{x}|\\
    &= |(\vect{A}_t^{1/2}\vect{\Delta})^\textrm{T} \vect{A}_t^{-1/2} \vect{x}|\\
    &\leq \| \vect{A}_t^{1/2}\vect{\Delta} \| \|\vect{A}_t^{-1/2} \vect{x}\|\\
    &= \| \vect{\Delta} \|_{\vect{A}_t} \| \vect{x} \|_{\vect{A}_t^{-1}}\\
    &= w \| \vect{\Delta} \|_{2, E_t}\\
    &\leq w \sqrt{\beta_t} 
\end{align*}
\end{proof}

Define the widths of the allocations as
\begin{equation}
    w_t := \| \vect{x}_t \|_{\vect{A}_t^{-1}} \qquad \text{and} \qquad w_{t, i} := \| \vect{e}_i \|_{\vect{A}_t^{-1}}
\label{def_widths}
\end{equation}

\begin{lemma}
For any $t \in \mathbb{N}$, we have the identity
\begin{equation*}
    w_t^2 = \sum_{i \in \mathcal{A}_t} w_{t, i}^2
\end{equation*}
\label{lemma_width_identity}
\end{lemma}
\begin{proof}
\begin{align*}
    w_t^2 &= \langle \vect{x}_t, \vect{x}_t \rangle_{\vect{A}_t^{-1}}\\
    &= \left\langle \vect{A}_t^{-1} \sum_{i \in \mathcal{A}_t} \vect{e}_i, \sum_{i \in \mathcal{A}_t} \vect{e}_i \right\rangle \\
    &= \sum_{i \in \mathcal{A}_t} \sum_{j \in \mathcal{A}_t} \left\langle \vect{A}_t^{-1} \vect{e}_i, \vect{e}_j \right\rangle\\
    &= \sum_{i \in \mathcal{A}_t} \left\langle \vect{A}_t^{-1} \vect{e}_i, \vect{e}_i \right\rangle\\
    &= \sum_{i \in \mathcal{A}_t} w_{t, i}^2
\end{align*}
where the penultimate step follows because $\left\langle \vect{A}_t^{-1} \vect{e}_i, \vect{e}_j \right\rangle = 0$ for $i \neq j$.
\end{proof}

\begin{lemma}
Let the regret at time $t$ be $r_t = \langle \vect{x}_t^*, \vect{\theta}^* \rangle - \langle \vect{x}_t, \vect{\theta}^* \rangle$. If $\vect{\theta}^* \in \mathcal{C}_t$, then
\begin{equation*}
    r_t \leq 2 w_t \sqrt{\beta_t}
\end{equation*}
\label{lemma_regret_ub}
\end{lemma}

\begin{proof}
By the choice of $(\vect{x}_t, \widetilde{\vect{\theta}}_t)$, we have
\begin{align*}
    \langle \vect{x}_t, \widetilde{\vect{\theta}}_t \rangle = \max_{(\vect{x}, \vect{\theta}) \in \mathcal{X}_t \times \mathcal{C}_t} \langle \vect{x}, \vect{\theta} \rangle \geq \langle \vect{x}_t^*, \vect{\theta}^* \rangle
\end{align*}
where the inequality uses $\vect{\theta}^* \in \mathcal{C}_t$. Hence,
\begin{align*}
    r_t &= \langle \vect{x}_t^*, \vect{\theta}^* \rangle - \langle \vect{x}_t, \vect{\theta}^* \rangle\\
    &\leq \langle \vect{x}_t, \widetilde{\vect{\theta}}_t - \vect{\theta}^* \rangle\\
    &= \langle \vect{x}_t, \widetilde{\vect{\theta}}_t - \widehat{\vect{\theta}}_t^{\text{LS}} \rangle + \langle \vect{x}_t, \widehat{\vect{\theta}}_t^{\text{LS}} - \vect{\theta}^* \rangle\\
    &\leq  2 w_t \sqrt{\beta_t}
\end{align*}

where the last step follows from Lemma $\ref{lemma_conf_width}$.
\end{proof}

Next, we show that the confidence widths do not grow too fast.

\begin{lemma}
For every t,
\begin{equation}
    \log \det \vect{A}_{t+1} = d \log \gamma + \sum_{\tau = 1}^{t} \sum_{i \in \mathcal{A}_\tau} \log(1 + w_{\tau,i}^2) 
\end{equation}
\label{lemma_logdet_multiply}
\end{lemma}

\begin{proof} By the definition of $\vect{A}_t$, we have
\begin{align*}
    \det \vect{A}_{t+1} &= \det \left( \vect{A}_{t} + \sum_{i \in \mathcal{A}_t} \vect{e}_{i} \vect{e}_{i}^\textrm{T} \right)\\
    &= \det \left( \vect{A}_{t}^{1/2} \left( \vect{I} + \vect{A}_{t}^{-1/2}\left( \sum_{i \in \mathcal{A}_t} \vect{e}_{i} \vect{e}_{i}^\textrm{T} \right) \vect{A}_{t}^{-1/2} \right) \vect{A}_{t}^{1/2} \right) \\
    &= \det (\vect{A}_{t}) \det\left( \vect{I} + \sum_{i \in \mathcal{A}_t} \vect{A}_{t}^{-1/2}  \vect{e}_{i} \vect{e}_{i}^\textrm{T} \vect{A}_{t}^{-1/2}  \right)
\end{align*}

Each $\vect{A}_{t}^{-1/2}  \vect{e}_{i} \vect{e}_{i}^\textrm{T} \vect{A}_{t}^{-1/2}$ term has zeros everywhere except one entry on the diagonal and that non-zero entry is equal to $w_{t, i}^2$. Furthermore, the location of the non-zero entry is different in for each term. Hence,
\begin{equation*}
    \det\left( \vect{I} + \sum_{i \in \mathcal{A}_t} \vect{A}_{t}^{-1/2}  \vect{e}_{i} \vect{e}_{i}^\textrm{T} \vect{A}_{t}^{-1/2}  \right) = \prod_{i \in \mathcal{A}_t} (1 + w_{t, i}^2)
\end{equation*}

Therefore, we have
\begin{equation*}
    \log \det \vect{A}_{t+1} = \log\det \vect{A}_{t} + \sum_{i \in \mathcal{A}_t} \log(1 + w_{t, i}^2)
\end{equation*}
Since $\vect{A}_1 = \gamma \vect{I}$, we have $\log \det \vect{A}_1 = d \log \gamma$ and the result follows by induction.
\end{proof}

\begin{lemma}
For all t, $\log \det \vect{A}_t \leq d \log (t + \gamma - 1)$.
\label{lemma_logdet_upper}
\end{lemma}
\begin{proof}
Noting that $\vect{A}_t$ is a diagonal matrix with diagonals $(n_{t, i} + \gamma)$,
\begin{align*}
    \tr \vect{A}_t &= d \gamma + \sum_{i = 1}^{d} n_{t, i} \\
    &= d \gamma + d (t-1) \\
    &= d(t + \gamma - 1)
\end{align*}
Now, recall that $\tr \vect{A}_t$ equals the sum of the eigenvalues of $\vect{A}_t$. On the other hand, $\det(\vect{A}_t)$ equals the product of the eigenvalues. Since $\vect{A}_t$ is positive definite, its eigenvalues are all positive. Subject to these constraints, $\det(\vect{A}_t)$ is maximized when all the eigenvalues are equal; the desired bound follows. 
\end{proof}

\begin{lemma}
Let $\gamma \geq 1$. Then, for all t, we have
\begin{equation*}
    \sum_{\tau = 1}^{t} \sum_{i \in \mathcal{A}_\tau} w_{\tau, i}^2 \leq 2 d \log \left(1 + \frac{t}{\gamma} \right)
\end{equation*}
\label{lemma_width_sum_ub}
\end{lemma}

\begin{proof}
Note that $0 \leq w_{\tau, i}^2 \leq 1$, if $\gamma \geq 1$. Using the inequality $y \leq 2 \log(1 + y)$ for $0 \leq y \leq 1$, we have
\begin{align*}
    \sum_{\tau = 1}^{t} \sum_{i \in \mathcal{A}_\tau} w_{\tau, i}^2 &\leq 2 \sum_{\tau = 1}^{t} \sum_{i \in \mathcal{A}_\tau} \log (1 + w_{\tau, i}^2)\\
    &= 2 \log \det \vect{A}_{t+1} - 2 d \log \gamma\\
    &\leq 2 d \log \left(1 + \frac{t}{\gamma} \right)
\end{align*}
where the last two lines follow from Lemmas \ref{lemma_logdet_multiply} and \ref{lemma_logdet_upper} respectively.
\end{proof}

\begin{lemma}
Let the instantaneous regret at time $t$ be $r_t = \langle \vect{x}_t^*, \vect{\theta}^* \rangle - \langle \vect{x}_t, \vect{\theta}^* \rangle$. If $\gamma \geq 1$ and $\vect{\theta}^* \in \mathcal{C}_t$ for all $t \leq T$, then
\begin{equation*}
    \sum_{t = 1}^{T} r_t^2 \leq 8 \beta_T d \log \left(1 + \frac{T}{\gamma} \right)
\end{equation*}
\label{lemma_square_regret}
\end{lemma}

\begin{proof}
 Assuming that $\vect{\theta}^* \in \mathcal{C}_t$ for all $t \leq T$,
\begin{align*}
    \sum_{t = 1}^{T} r_t^2 &\leq  \sum_{t = 1}^{T} 4 w_t^2 \beta_t\\
    &\leq 4 \beta_T \sum_{t = 1}^{T} w_t^2\\
    &= 4 \beta_T \sum_{t = 1}^{T} \sum_{i \in \mathcal{A}_t}  w_{t, i}^2 \\
    &\leq 8 \beta_T d \log \left(1 + \frac{T}{\gamma} \right)
\end{align*}
where the first step follows from Lemma \ref{lemma_regret_ub}, second step follows from the definition of $\beta_t$, third step uses the identity given in Lemma \ref{lemma_width_identity} and the last step is due to Lemma \ref{lemma_width_sum_ub}.
\end{proof}

\begin{theorem}
Let $\gamma \geq 1$. Then, with probability at least $1 - 2 \delta$, the T period regret is bounded by
\begin{equation*}
    \mathcal{R}(T, \pi) \leq \sqrt{ 8 d \beta_T^* (\delta, \alpha, \gamma) T \log \left(1 + \frac{T}{\gamma} \right) }
\end{equation*}
where
\begin{equation}
    \beta_T^*(\delta, \alpha, \gamma) = 8 \eta^2 \log \left(\mathcal{N}(\mathcal{F}, \alpha, \| \cdot \|_{2}) / \delta \right) + 2 \alpha d T \left(8 B + \sqrt{8 \eta^2 \log(4 d T^2  / \delta)} \right) + 4 \gamma G^2
\end{equation}
\label{thm_regret}
\end{theorem}

\begin{proof}
Assuming that $\vect{\theta}^* \in \mathcal{C}_t$ for all $t \leq T$,
\begin{align*}
    \mathcal{R}(T, \pi) &= \sum_{t = 1}^{T} r_t \\
    &\leq \left( T \sum_{t = 1}^{T} r_t^2 \right)^{1/2}\\
    &\leq \left( 8 d \beta_T T \log \left(1 + \frac{T}{\gamma} \right) \right)^{1/2}
\end{align*}
where the last step follows from Lemma \ref{lemma_square_regret}. Then, by Lemma \ref{lemma_confidence}, $\vect{\theta}^* \in \mathcal{C}_t$ for all $t \leq T$ with probability at least $1 - 2 \delta$. Therefore, the bound holds true with probability at least $1 - 2 \delta$.
\end{proof}

\begin{corollary}
Letting $\delta = \mathcal{O} ((dT)^{-1})$, $\alpha = \mathcal{O} ((dT)^{-1})$ and $\gamma = 1$ results in a regret bound that satisfies
\begin{equation}
    \mathcal{R}(T, \pi) = \widetilde{\mathcal{O}} \left( \sqrt{\eta^2 d \log \left(\mathcal{N}(\mathcal{F}, T^{-1}, \| \cdot \|_{2})\right) T} \right)
\end{equation}
\label{corr_regret_order}
\end{corollary}

\begin{proof}
By Theorem \ref{thm_regret}, with probability $1$,
\begin{equation*}
    \mathcal{R}(T, \pi) \leq (1 - \delta) \sqrt{ 8 d \beta_T^* (\delta, \alpha, \gamma) T \log \left(1 + \frac{T}{\gamma} \right) } + 2 \delta BdT
\end{equation*}

Letting $\delta = \mathcal{O} (T^{-1})$, $\alpha = \mathcal{O} = (T^{-1})$ and $\gamma = 1$,
\begin{equation*}
    \mathcal{R}(T, \pi) = \widetilde{\mathcal{O}} \left( \sqrt{d \beta_T^* (T, T^{-1}, 1) T} \right)
\end{equation*}

Noting that $\beta_T^* (T, T^{-1}, 1) = \widetilde{\mathcal{O}} \left(\eta^2 \log \left(\mathcal{N}(\mathcal{F}, T^{-1}, \| \cdot \|_{2})\right) \right)$, the proof is complete.

\end{proof}

\section{Proofs for OFU-based Allocations}
\label{pf_allocation}

\subsection{Proof of Theorem \ref{thm_alloc_regret}}

In the allocation problem, the mean reward of the arms are given in the matrix $\vect{\Theta} \in \mathbb{R}^{N \times M}$. Consider setting $\vect{\theta} = \text{vec}({\vect{\Theta}})$ as the mean reward vector for the problem described in Appendix Section \ref{SCMAB}. Noting that $d = NM$ and $\|\cdot\|_\text{F} = \|\text{vec}(\cdot)\|_2$, the proof becomes a direct extension of Theorem \ref{thm_regret}.

\subsection{Proof of Theorem  \ref{low_rank_regret_thm}}

The proof is direct extension of Corollary \ref{corr_regret_order} where the covering number is replaced with the following upper bound given for the choice of $\mathcal{L}$ defined in equation \eqref{low_l}. We provide the upper bound for the covering number of $\mathcal{L}$ in the following lemma.

\begin{lemma}[Covering number for low-rank matrices]
The covering number of $\mathcal{L}$ given in \eqref{low_l} obeys
\begin{equation}
    \log \mathcal{N}(\mathcal{L}, \alpha, \| \cdot \|_{\text{F}}) \leq (N + M + 1) R \log \left( \frac{9B \sqrt{NM}}{\alpha} \right)
\end{equation}
\label{lemma_covering}
\end{lemma}

\begin{proof}
This proof is modified from \cite{candes_2011}. Let $\mathcal{S} = \{ \vect{\Theta} \in \mathds{R}^{N \times M} : \text{rank}(\vect{\Theta}) \leq R, \|\vect{\Theta}\|_\text{F} \leq 1\}$. We will first show that there exists an $\epsilon$-net $\mathcal{S}^\epsilon$ for the Frobenious norm obeying 
\begin{equation*}
    |\mathcal{S}^\epsilon| \leq \left(9 / \epsilon \right)^{(N+M+1)R}
\end{equation*} 

For any $\vect{\Theta} \in \mathcal{S}$, singular value decomposition gives $\vect{\Theta} = \vect{U} \vect{\Sigma} \vect{V}^\textrm{T}$, where $\|\vect{\Sigma}\|_{\text{F}} \leq 1$. We will construct an $\epsilon$-net for $\mathcal{S}$ by covering the set of permisible $\vect{U}$, $\vect{\Sigma}$ and $\vect{V}$. Let $\mathcal{D}$ be the set of diagonal matrices with nonnegative diagonal entries and Frobenious norm less than or equal to one. We take $\mathcal{D}^{\epsilon/3}$ be an $\epsilon$/3-net for $\mathcal{D}$ with $|\mathcal{D}^{\epsilon/3}| \leq (9/\epsilon)^R$. Next, let $\mathcal{O}_{N,R} = \{\vect{U} \in \mathbb{R}^{N \times R} : \vect{U}^\textrm{T} \vect{U} = \vect{I} \}$. To cover $\mathcal{O}_{N,R}$, we use the $\|\cdot\|_{1,2}$ norm defined as
\begin{equation*}
    \|\vect{U}\|_{1,2} = \max_{i} \|\vect{u_i}\|_{\ell_2}
\end{equation*}
where $\vect{u_i}$ denotes the $i$th column of $\vect{U}$. Let $\mathcal{Q}_{N, R} = \{\vect{U} \in \mathbb{R}^{N, R} : \|\vect{U}\|_{1,2} \leq 1\}$. It is easy to see that $\mathcal{O}_{{N, R}} \subset \mathcal{Q}_{{N, R}}$ since the columns of an orthogonal matrix are unit normed. We see that there is an $\epsilon/3$-net $\mathcal{O}_{{N, R}}^{\epsilon/3}$ for $\mathcal{O}_{{N, R}}$ obeying $|\mathcal{O}_{{N, R}}^{\epsilon/3}| \leq (9/\epsilon)^{NR}$. Similarly, let $\mathcal{P}_{M,R} = \{\vect{V} \in \mathbb{R}^{M \times R} : \vect{V}^\textrm{T} \vect{V} = \vect{I} \}$.  By the same argument, there is an $\epsilon/3$-net $\mathcal{P}_{M, R}^{\epsilon/3}$ for $\mathcal{P}_{M, R}$ obeying $|\mathcal{P}_{{M, R}}^{\epsilon/3}| \leq (9/\epsilon)^{MR}$. We now let $\mathcal{S}^\epsilon = \{ \bar{\vect{U}} \bar{\vect{\Sigma}} \bar{\vect{V}}^\textrm{T} : \bar{\vect{U}} \in \mathcal{O}_{{N, R}}^{\epsilon/3}, \bar{\vect{V}} \in \mathcal{P}_{{M, R}}^{\epsilon/3}, \bar{\vect{\Sigma}} \in \mathcal{D}^{\epsilon/3}\}$, and remark $|\mathcal{S}^\epsilon| \leq |\mathcal{O}_{{N, R}}^{\epsilon/3}| |\mathcal{P}_{{M, R}}^{\epsilon/3}| |\mathcal{D}^{\epsilon/3}| \leq (9/\epsilon)^{(N+M+1)R}$. It remains to show that for all $\vect{\Theta} \in \mathcal{S}$, there exists $\bar{\vect{\Theta}} \in \mathcal{S}^\epsilon$ with $\|\vect{\Theta} - \bar{\vect{\Theta}}\|_\text{F} \leq \epsilon$.

Fix $\vect{\Theta} \in \mathcal{S}$ and decompose it as $\vect{\Theta} = \vect{U} \vect{\Sigma} \vect{V}^\textrm{T}$. Then, there exists $\bar{\vect{\Theta}} = \bar{\vect{U}} \bar{\vect{\Sigma}} \bar{\vect{V}}^\textrm{T} \in \mathcal{S}^\epsilon$ with $\bar{\vect{U}} \in \mathcal{O}_{{N, R}}^{\epsilon/3}, \bar{\vect{V}} \in \mathcal{P}_{{M, R}}^{\epsilon/3}, \bar{\vect{\Sigma}} \in \mathcal{D}^{\epsilon/3}$ satisfying $\|\vect{U} - \bar{\vect{U}}\|_{1, 2} \leq \epsilon/3$, $\|\vect{V} - \bar{\vect{V}}\|_{1, 2} \leq \epsilon/3$ and $\|\vect{\Sigma} - \bar{\vect{\Sigma}}\|_{\text{F}} \leq \epsilon/3$. This gives 
\begin{align*}
    \|\vect{\Theta} - \bar{\vect{\Theta}}\|_{\text{F}} &= \|\vect{U} \vect{\Sigma} \vect{V}^\textrm{T} - \bar{\vect{U}} \bar{\vect{\Sigma}} \bar{\vect{V}}^\textrm{T}\|_{\text{F}}\\
    &= \|\vect{U} \vect{\Sigma} \vect{V}^\textrm{T} - \bar{\vect{U}} \vect{\Sigma} \vect{V}^\textrm{T} + \bar{\vect{U}} \vect{\Sigma} \vect{V}^\textrm{T} - \bar{\vect{U}} \bar{\vect{\Sigma}} \vect{V}^\textrm{T} + \bar{\vect{U}} \bar{\vect{\Sigma}} \vect{V}^\textrm{T} - \bar{\vect{U}} \bar{\vect{\Sigma}} \bar{\vect{V}}^\textrm{T}\|_{\text{F}} \\
    &\leq \|(\vect{U} - \bar{\vect{U}}) \vect{\Sigma} \vect{V}^\textrm{T}\|_{\text{F}} + \|\bar{\vect{U}} (\vect{\Sigma} - \bar{\vect{\Sigma}}) \vect{V}^\textrm{T}\|_{\text{F}} + \|\bar{\vect{U}} \bar{\vect{\Sigma}} (\vect{V} - \bar{\vect{V}})^\textrm{T} \|_{\text{F}}
\end{align*}

For the first term, since $\vect{V}$ is an orthogonal matrix,
\begin{align*}
    \|(\vect{U} - \bar{\vect{U}}) \vect{\Sigma} \vect{V}^\textrm{T}\|^2_{\text{F}} &= \|(\vect{U} - \bar{\vect{U}}) \vect{\Sigma} \|^2_{\text{F}} \\
    &\leq \|\Sigma\|^2_{\text{F}} \|\vect{U} - \bar{\vect{U}} \|^2_{1, 2} \leq (\epsilon/3)^2
\end{align*}
By the same argument, $\|\bar{\vect{U}} \bar{\vect{\Sigma}} (\vect{V} - \bar{\vect{V}})^\textrm{T} \|_{\text{F}} \leq \epsilon/3$ as well. Lastly, $\|\bar{\vect{U}} (\vect{\Sigma} - \bar{\vect{\Sigma}}) \vect{V}^\textrm{T}\|_{\text{F}} = \|\vect{\Sigma} - \bar{\vect{\Sigma}}\|_{\text{F}} \leq \epsilon / 3$. Therefore, $\|\vect{\Theta} - \bar{\vect{\Theta}}\|_{\text{F}} \leq \epsilon$, showing that $\mathcal{S}^\epsilon$ is an $\epsilon$-net for $\mathcal{S}$ with respect to the Frobenious norm.

Next, we will construct an $\alpha$-net for $\mathcal{L}$ given in equation \ref{low_l}. Let $\kappa = B\sqrt{NM}$. We start by noting that for all $\vect{\Theta} \in \mathcal{L}$, the Frobenious norm obeys $\|\vect{\Theta}\|_\text{F} \leq \kappa$. Then, define $\vect{X} = \frac{1}{\kappa} \vect{\Theta} \in \mathcal{S}$ and $\mathcal{L}^\alpha := \left\{ \kappa \bar{\vect{X}} : \bar{\vect{X}} \in \mathcal{S}^\epsilon \right\}$. We previously showed that for any $\vect{X} \in \mathcal{S}$,  there exists $\bar{\vect{X}} \in \mathcal{S}^\epsilon$ such that $\|\vect{X} - \bar{\vect{X}}\|_\text{F} \leq \epsilon$. Therefore, for any $\vect{\Theta} \in \mathcal{L}$,  there exists $\bar{\vect{\Theta}} = \kappa \bar{\vect{X}} \in \mathcal{L}^\alpha$ such that $\|\vect{\Theta} - \bar{\vect{\Theta}}\|_\text{F} \leq \kappa \epsilon$. Setting $\epsilon = \alpha / \kappa$, we obtain that $\mathcal{L}^\alpha$ is an $\alpha$-net for $\mathcal{L}$ with respect to the Frobenious norm. Finally, the size of $\mathcal{L}^\alpha$ obeys
\begin{equation*}
    |\mathcal{L}^\alpha| = |\mathcal{S}^{\alpha / \kappa}| \leq \left(9 \kappa / \alpha \right)^{(N+M+1)R}
\end{equation*} 
This completes the proof.
\end{proof}

\section{Additional Experimental Details}
\label{sect_additional_exp}

All experiments are implemented in Python and carried out on a machine with 2.3GHz 8-core Intel Core i9 CPU and 16GB of RAM. We solve the allocation integer program \eqref{integer_num} using large-scale mixed integer programming (MIP) solver packages to have efficient computations.

\textbf{Parameter setup}:
\begin{itemize}[labelindent= 0pt, align= left, labelsep=0.4em, leftmargin=*]
\item In synthetic data, $\vect{\Theta}^*$ is scaled such that $B = 10$.
\item Standard deviation of the rewards: $\eta = 1$.
\item In the static setting, $d_{t, u} = 1$ for all $u \in [N]$.
\item In the dynamic setting, $d_{t, u} = 1$ with probability $0.2$, $0$ otherwise, independently for each $u \in [N]$.
\item $C_\text{max} = \text{ceil} \left( \frac{3}{M} \sum_{u = 1}^{N} d_{t, u} \right)$.
\item $c_{t, i}$ are uniformly sampled over $\{1, \dots, C_\text{max}\}$ independently for each $t \in [T]$ and $i \in [M]$.

\end{itemize}

To complement our discussion on the importance of capacity-aware recommendations, Figure \ref{fig_cap} shows the effects of having stricter capacity constraints. When $C_\text{max}$ is low (the capacity constraints are stricter), we see that the performance gap between our proposed method and ICF2 is larger. Therefore, it is more crucial to employ capacity-aware mechanisms in settings with tighter capacity constraints. 

\begin{figure}[ht]
\center
\includegraphics[width=0.4\textwidth]{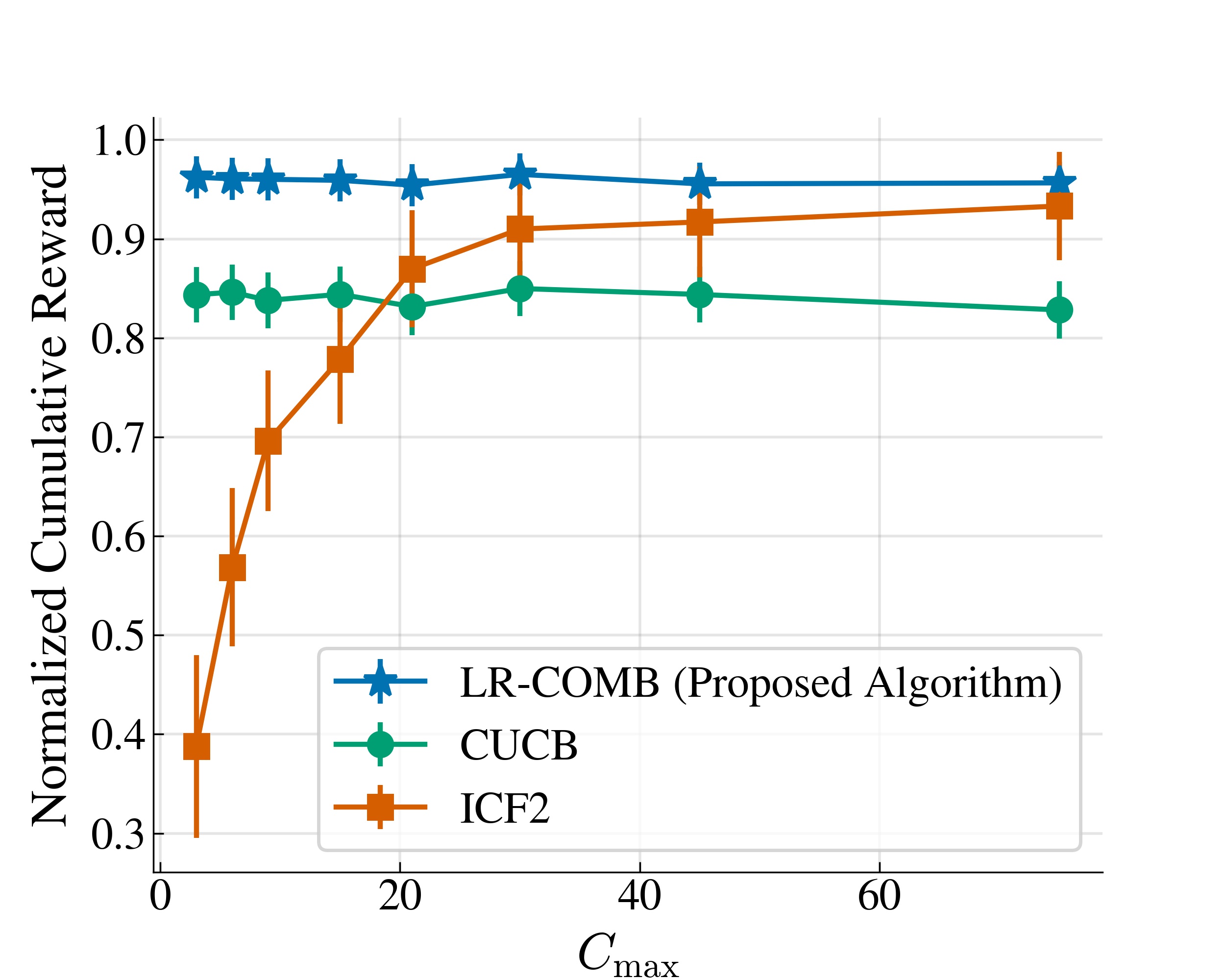}
\caption{Normalized cumulative reward obtained in $T = 300$ rounds for different choices of $C_{\text{max}}$ (normalized by the cumulative reward of optimal allocations). Synthetic data in a static setting with $N = 400$, $M = 200$, $R = 10$. For each data point, the experiments are run on $20$ problem instances and means are reported together with error regions that indicate one standard deviation of uncertainty.}
\label{fig_cap}
\end{figure}

In Figures \ref{exp_1}, \ref{exp_2}, \ref{exp_3} and \ref{exp_4}, we provide detailed  results for different experimental settings described in Section \ref{sect_exp}. Reward indicates the instantaneous reward obtained in each iteration, regret is the gap between the reward of the optimum allocation and the allocation achieved by the algorithm. Cumulative regret (defined in \eqref{cumulative_regret}) is the cumulative sum of instantaneous regrets up to iteration $t$.  The average cumulative regret is obtained by normalizing the cumulative regret with $1 / t$.

\begin{figure}[ht]
\center
\includegraphics[width=\textwidth]{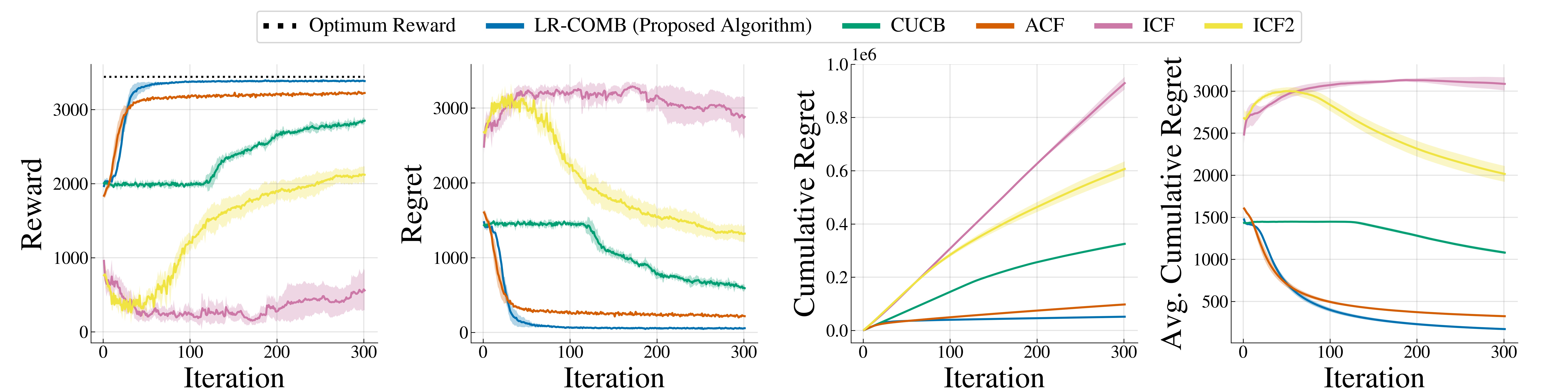}
\caption{Experimental results for synthetic data in a static setting with $N = 800$, $M = 400$, $R = 20$. The experiments are run on $10$ problem instances and means are reported together with error regions that indicate one standard deviation of uncertainty.}
\label{exp_1}
\end{figure}

\begin{figure}[ht]
\center
\includegraphics[width=\textwidth]{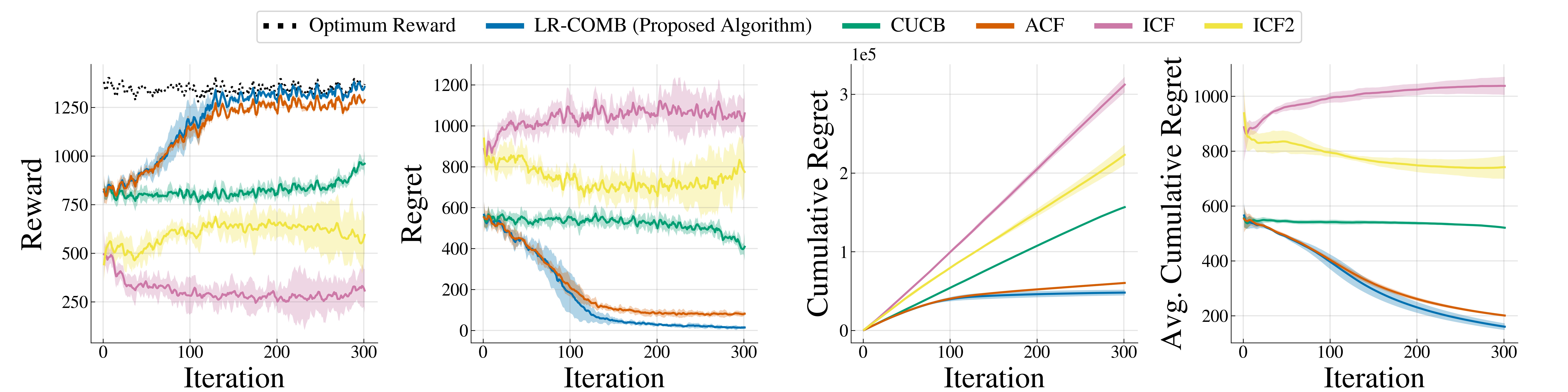}
\caption{Experimental results for synthetic data in a dynamic setting with $N = 1000$, $M = 150$, $R = 20$, probability of activity $0.2$. The experiments are run on $10$ problem instances and means are reported together with error regions that indicate one standard deviation of uncertainty.}
\label{exp_2}
\end{figure}

\begin{figure}[ht]
\center
\includegraphics[width=\textwidth]{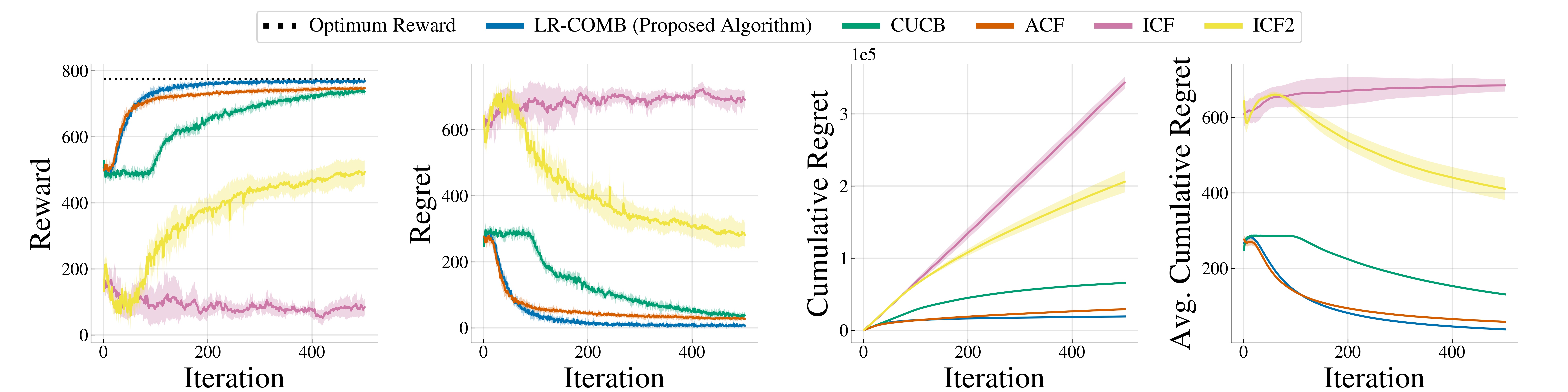}
\caption{Experimental results for Restaurant-Customer data in a static setting. The experiments are run on $10$ problem instances and means are reported together with error regions that indicate one standard deviation of uncertainty.}
\label{exp_3}
\end{figure}

\begin{figure}[ht]
\center
\includegraphics[width=\textwidth]{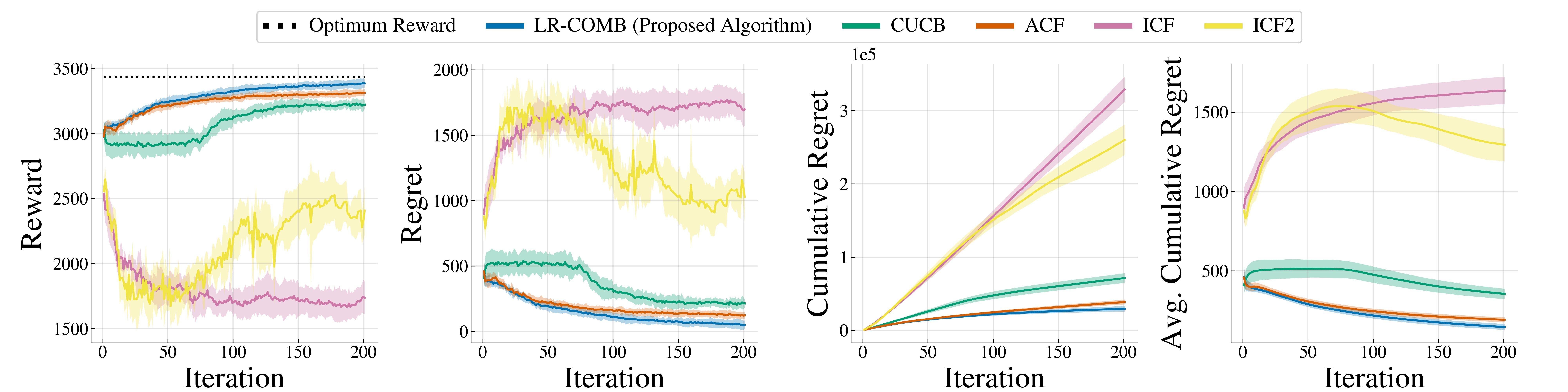}
\caption{Experimental results for MovieLens 100k data in a static setting. The experiments are run on $10$ problem instances and means are reported together with error regions that indicate one standard deviation of uncertainty.}
\label{exp_4}
\end{figure}

\end{document}